\documentclass{article}

\usepackage{microtype}
\usepackage{graphicx}
\usepackage{subcaption}
\usepackage{booktabs}
\usepackage{caption}
\usepackage[accepted]{icml2026}

\usepackage{hyperref}

\usepackage{amsmath}
\allowdisplaybreaks
\usepackage{amssymb}
\usepackage{mathtools}
\usepackage{amsthm}
\usepackage{pifont}
\usepackage{tikz}
\usetikzlibrary{positioning,calc}
\usetikzlibrary{arrows.meta}

\usepackage{bm}
\usepackage{xspace}
\usepackage{empheq}
\usepackage{nicefrac}

\usepackage[textsize=tiny]{todonotes}

\usepackage{aliascnt}
\theoremstyle{plain}
\newtheorem{theorem}{Theorem}[section]

\newaliascnt{proposition}{theorem}

\aliascntresetthe{proposition}

\newaliascnt{lemma}{theorem}

\aliascntresetthe{lemma}

\newaliascnt{corollary}{theorem}

\aliascntresetthe{corollary}

\theoremstyle{definition}
\newaliascnt{definition}{theorem}

\aliascntresetthe{definition}

\newaliascnt{assumption}{theorem}
\newtheorem{assumption}[assumption]{Assumption}
\aliascntresetthe{assumption}

\theoremstyle{remark}
\newaliascnt{remark}{theorem}

\aliascntresetthe{remark}

\usepackage{enumitem}

\usepackage[capitalize,noabbrev]{cleveref}
\crefname{table}{Table}{Tables}
\Crefname{table}{Table}{Tables}

\crefname{theorem}{Theorem}{Theorems}
\Crefname{theorem}{Theorem}{Theorems}

\crefname{lemma}{Lemma}{Lemmas}
\Crefname{lemma}{Lemma}{Lemmas}

\crefname{proposition}{Proposition}{Propositions}
\Crefname{proposition}{Proposition}{Propositions}

\crefname{corollary}{Corollary}{Corollaries}
\Crefname{corollary}{Corollary}{Corollaries}

\crefname{definition}{Definition}{Definitions}
\Crefname{definition}{Definition}{Definitions}

\crefname{assumption}{Assumption}{Assumptions}
\Crefname{assumption}{Assumption}{Assumptions}

\crefname{remark}{Remark}{Remarks}
\Crefname{remark}{Remark}{Remarks}

\usepackage{xcolor}
\definecolor{tabblue}{RGB}{31,119,180}
\definecolor{tabred}{RGB}{214,39,40}
\definecolor{taborange}{RGB}{255,127,14}   
\definecolor{tabgreen}{RGB}{44,160,44} 
\usetikzlibrary{decorations.pathreplacing}

\newcommand{\nbone}{\ding{182}\xspace}
\newcommand{\nbtwo}{\ding{183}\xspace}
\newcommand{\nbthree}{\ding{184}\xspace}

\newcommand{\proximh}{Proximal-IMH}
\newcommand{\latentimh}{Latent-IMH}
\newcommand{\approximh}{Approx-IMH}

\newcommand{\ApproxMH}{{\texttt{\approximh}}\xspace}
\newcommand{\LatentMH}{{\texttt{\latentimh}}\xspace}
\newcommand{\ProxMH}{{\texttt{\proximh}}\xspace}

\newcommand{\Exact}{{\texttt{Exact posterior}}\xspace}
\newcommand{\Approx}{{\texttt{Approx posterior}}\xspace}
\newcommand{\Latent}{{\texttt{Latent posterior}}\xspace}
\newcommand{\Proximal}{{\texttt{Proximal posterior}}\xspace}

\newcommand{\piapprox}{\ensuremath{\pi_a}}
\newcommand{\pinew}{\ensuremath{\pi_l}}
\newcommand{\pilatent}{\ensuremath{\pi_l}}
\newcommand{\piprox}{\ensuremath{\pi_p}}

\newcommand{\KL}{{\mathrm{KL}\xspace}}

\newcommand{\KLapprox}{\mathbb{D}_a}

\newcommand{\KLlatent}{\mathbb{D}_l}
\newcommand{\KLprox}{\mathbb{D}_p}

\newcommand{\mixapprox}{\tau_{mix}^a}
\newcommand{\mixlatent}{\tau_{mix}^l}
\newcommand{\mixprox}{\tau_{mix}^p}

\renewcommand{\b}[1]{{\bf #1}}

\newcommand{\A}{\b{A}}
\newcommand{\J}{\b{J}}

\newcommand{\M}{\b{M}}

\newcommand{\V}{\b{V}}

\newcommand{\Sapprox}{\widetilde{\b{S}}}
\newcommand{\daapprox}{\mathbf{\Delta}_a}
\newcommand{\daprox}{\mathbf{\Delta}_p}

\newcommand{\xapprox}{\widetilde{x}}
\newcommand{\xnew}{x_l}
\newcommand{\wapprox}{w_a}
\newcommand{\wnew}{w_l}
\newcommand{\wprox}{w_p}

\newcommand{\w}{\b{w}}

\newcommand{\W}{\b{W}}

\newcommand{\Oo}{\b O}
\newcommand{\F}{\b F}
\newcommand{\K}{\b K}

\newcommand{\Id}{\b I}

\newcommand{\Fapprox}{\widetilde{\F}}

\newcommand{\Aapprox}{\widetilde{\A}}
\newcommand{\Anew}{\A_l}

\DeclareMathOperator{\Tr}{Tr}

\DeclareMathOperator*{\argmin}{arg\,min}

\definecolor{mygreen}{RGB}{0,153,0}
\definecolor{light-gray}{gray}{0.93}
\definecolor{mid-gray}{gray}{0.88}

\newcommand{\bSigma}{\bm{\Sigma}}

\makeatletter
\newcommand*{\rom}[1]{\expandafter\@slowromancap\romannumeral #1@}
\makeatother

\newcommand{\apu}[1]{\ensuremath{\widetilde{p}(#1)}}

\begin{document}
\raggedbottom

\twocolumn[
  \icmltitle{Proximal-IMH: Proximal Posterior Proposals for Independent Metropolis--Hastings with Approximate Operators}

  \begin{icmlauthorlist}
    \icmlauthor{Youguang Chen}{oden}
    \icmlauthor{George Biros}{oden}
  \end{icmlauthorlist}

  \icmlaffiliation{oden}{Oden Institute for Computational Engineering and Sciences, The University of Texas at Austin, Austin, Texas, USA}

  \icmlcorrespondingauthor{Youguang Chen}{youguang@utexas.edu}

  \icmlkeywords{Bayesian Inference, Inverse Problems, Markov Chain Monte Carlo, Independent Metropolis--Hastings, Approximate Operators}

  \vskip 0.3in
]



\printAffiliationsAndNotice{}  

\begin{abstract}
We are considering the problem of sampling from a posterior distribution related to Bayesian inverse problems arising in science, engineering, and imaging. Our method belongs to the family of  independence Metropolis–Hastings (IMH) sampling algorithms. These are quite common in Bayesian inference.  Relying on the existence of an approximate posterior distribution that is cheaper to  sample from but can have significant bias, we introduce Proximal-IMH, a scheme that removes this bias: it corrects samples from the approximate posterior   solving an auxiliary optimization problem, yielding a local adjustment that trades off adherence to the exact model against stability around the approximate reference point. For idealized settings, we prove  that  the proximal correction tightens the match between approximate and exact posteriors, and thereby improves acceptance rates and mixing. The new method works with both linear and nonlinear input-output operators and is especially suitable for inverse problems where exact posterior sampling is too expensive. We perform several numerical experiments that include multimodal and data-driven priors and  nonlinear input-output operators.  The results show that Proximal-IMH reliably outperforms existing IMH variants.
\end{abstract}

\section{Introduction}
We consider linear and nonlinear inverse problems of the form \begin{equation}\label{e:forward}
y = \Oo \F(x) + e = \A(x) + e,
\end{equation}
where $x \in \mathbb{R}^{d_x}$ denotes the unknown parameter vector with prior $p(x)$, $y \in \mathbb{R}^{d_y}$ is the noise-corrupted observation, and
$e$ is observational noise with known distribution $q(e)$. Here, $\Oo$ is a linear \emph{observation operator}, $\F(x)$ is a (possibly nonlinear) \emph{forward operator}, and $\A = \Oo \F$ is the associated \emph{input--output operator}. The operator $\F$ typically corresponds to the solution map of a differential or integral equation driven by the parameter $x$, and $u=\F(x)$ is a latent variable representing this solution. As an example from acoustics, $x$ may describe an unknown
scatterer, $u$ is the resulting acoustic field---defined everywhere; and $y$ consists of measurements collected at sensor locations. The Bayesian posterior distribution is given by $\pi(x\mid y) \propto q(y-\Oo \F(x))\, p(x)$.
Our goal is to design fast and scalable algorithms for efficiently generating samples from $\pi(x\mid y)$.

Inverse problems of the form \cref{e:forward} arise in a wide range of applications, including tomography~\citep{kaipio-e00,saratoon-arridge-e13,arridge-schotland99}, medical imaging~\citep{liang2020deep,epstein2007introduction},
geophysics~\citep{tarantola-05}, and many others~\citep{book-kaipio2006,vogel02,ghattas2021learning}.
In many such problems, evaluating the forward operator $\F$ is computationally expensive, but a cheaper approximation $\Fapprox$ is available such that
$\Fapprox(x) \approx \F(x)$.
Defining $\Aapprox=\Oo \Fapprox$, this naturally leads to the \emph{approximate posterior} $\piapprox(x\mid y) \propto q(y-\Aapprox(x))\, p(x)$. Examples of $\Fapprox$ include preconditioners, truncated iterative solvers,
coarse-grid discretizations, and learned surrogate or neural operators \citep{cohen2014solving,van2003iterative,boomeramg00,herrmann2024neural}.

A broad class of methods has been developed for sampling from $\pi(x\mid y)$, including Markov chain Monte Carlo (MCMC) methods such as Langevin methods~\citep{girolami2011riemann}, NUTS~\citep{hoffman2014no}, and Metropolis adjusted Langevin algorithm (MALA)~\citep{mala-roberts1996}, generative models~\citep{inverse-song-ermon21}, normalizing flows~\citep{papamakarios2021normalizing}, low-rank approximations~\citep{spantini2015optimal}, and many others~\citep{biegler2010large}.

Several methods seek to exploit the availability of an approximate forward operator $\Fapprox$ to reduce computational cost. Among them, two approaches are particularly relevant to this work: Latent-IMH and multifidelity sampling. Multifidelity sampling methods, including delay-acceptance and two-stage algorithms~\citep{2stage}, use $\Fapprox$ to screen or precondition proposals from a local MCMC kernel.
While this reduces the number of expensive forward evaluations, the resulting schemes remain inherently local, with mixing governed by the underlying proposal.
Latent-IMH~\citep{chen-biros26} constructs a global independence proposal that
preserves the exact likelihood while modifying the prior via an approximate operator, enabling efficient global moves. However, it requires transforming the forward operator into a square, invertible form and constructing an approximate distribution in a high-dimensional latent space, which can lead to ill-conditioning. We review Latent-IMH in \cref{sec:background} and include it as a baseline in our numerical experiments.

In this paper, we introduce an IMH sampler that exploits $\Fapprox$ and addresses key limitations of \LatentMH. Our contributions are summarized below.
\nbone
We introduce the \ProxMH sampling method for inverse problems of the form \cref{e:forward} that exploits
a computationally cheap approximate operator $\Fapprox$, and applies to both linear (\cref{sec:linear}) and nonlinear (\cref{sec:gn}) forward models.
\nbtwo
For linear forward models, we provide a theoretical analysis showing that \ProxMH achieves smaller expected KL divergence (\cref{sec:kl}) and faster mixing times (\cref{sec:mix-time}) than existing IMH proposals under suitable regularity assumptions.
\nbthree
We demonstrate the effectiveness of \ProxMH on a range of linear and nonlinear inverse problems, including bimodal targets, imaging problems with learned priors, and nonlinear acoustic scattering (\cref{sec:results}). Across all tests, \ProxMH consistently outperforms other IMH schemes in terms of acceptance rates, convergence speed, and sampling accuracy.

%

\section{Background}\label{sec:background}
We begin by defining the exact posterior and two approximate posteriors.
\begin{subequations}\label{e:posteriors}
\begin{align}
\texttt{Exact:}\quad
& \pi(x\mid y) \propto q(y-\A x)\, p(x),
\label{eq:exact-post} \\[2pt]
\texttt{Approx:}\quad
& \piapprox(x\mid y) \propto q(y-\Aapprox x)\, p(x),
\label{eq:approx-post} \\[2pt]
\texttt{Latent:}\quad
& \pilatent(x\mid y) \propto q(y-\A x)\, p(\Fapprox^{-1}\F x),
\label{eq:latent-post}
\end{align}
\end{subequations}
where the formulation for \Latent assumes that $\F$ and $\Fapprox$ are linear and invertible. The approximate distributions $\piapprox$ and $\pilatent$ can be used as proposal densities $g(x\mid y)$ in the independence Metropolis--Hastings (IMH) algorithm (\cref{algo:imh}).

\begin{algorithm}[t]
\small
\caption{Independence Metropolis-Hastings (IMH)}
\label{algo:imh}
\textbf{In:} $y$, number of MH steps $k$, proposal $g(x\mid y)$\\
\textbf{Out:} Samples $\{x_i\}_{i=1}^{k+1}$ from $\pi(x\mid y)$
\begin{algorithmic}[1]
 \STATE $x_1 \sim g(x\mid y)$
 \FOR{$t=1$ to $k$}
 \STATE $x^\prime \sim g(x\mid y)$
 \STATE $a(x^\prime,x_t) = \min\!\left\{1,
 \frac{\pi(x^\prime\mid y)}{\pi(x_t\mid y)}
 \frac{g(x_t\mid y)}{g(x^\prime\mid y)}\right\}$
 \STATE Accept $x^\prime$ with probability $a(x^\prime,x_t)$
 \STATE \textbf{otherwise} set $x_{t+1} \gets x_t$
 \ENDFOR
\end{algorithmic}
\end{algorithm}

A key advantage of IMH over local MCMC methods is that proposals are independent of the current state. Consequently, the computationally expensive computations involving the forward operator $\F$—including those required for proposal construction and the
corresponding terms in the acceptance ratio—can be precomputed in advance and evaluated in an embarrassingly parallel fashion. In practice, this can lead to substantial wall-clock savings for large-scale inverse problems. The efficiency of IMH is therefore governed by the acceptance ratio: for a fixed
accuracy target, lower acceptance rates require more Metropolis–Hastings steps and hence more evaluations of the exact forward operator, increasing the overall
computational cost. Using \Approx as the proposal is computationally attractive, but for Gaussian noise the log-acceptance ratio scales as $1/\sigma^2$, causing small operator errors in $\Aapprox$ to be strongly amplified and leading to poor acceptance rates in practice.

The \LatentMH method~\cite{chen-biros26} addresses this issue by preserving the exact likelihood
$q(y-\A x)$ in the proposal distribution.
This is achieved by sampling in the latent variable $u=\Fapprox x$.
Specifically, samples are first drawn from an approximate latent prior $\apu{u}$ constructed offline, then conditioned on $y$ using only the inexpensive observation operator $\Oo$, and finally mapped back via $x=\F^{-1}u$.
This procedure is equivalent to using \cref{eq:latent-post} as the IMH proposal.
While effective, \LatentMH requires squaring $\F$ to enforce invertibility, which can introduce ill-conditioning, and relies on constructing the approximate distribution $\apu{u}$.






\section{Proximal Independent Metropolis-Hastings}
The motivation for \ProxMH is to provide an alternative to \LatentMH that avoids
ill-conditioning induced by squaring $\F$ and bypasses sampling in the latent
variable $u$, which can be high-dimensional, nonlinear, or degenerate in practice.
\ProxMH instead proceeds in two steps:
\begin{subequations}\label{eq:proximal}
\begin{empheq}[left=\empheqlbrace]{align}
\xapprox &\sim \piapprox(x \mid y), \label{eq:proximal-a}\\
 x &= \argmin_{x} \|\A(x) - \Aapprox (\xapprox)\|^2 + \beta\| x -\xapprox\|^2. \label{eq:proximal-b}
\end{empheq}
\end{subequations}

\subsection{Linear Input-Output Operator $\A$}\label{sec:linear}
When the forward operator $\A$ is linear and full rank, the optimization problem \cref{eq:proximal-b} admits a closed-form solution:
\begin{align}\label{eq:sol-linear}
    x = \K \xapprox, \qquad \K = (\A^\top \A + \beta \b I)^{-1}(\A^\top \Aapprox + \beta \b I).
\end{align}
The resulting proposal distribution is the push-forward $\piprox(\cdot \mid y)=\K_{\#}\piapprox(\cdot \mid y)$ and admits the density
\begin{align}\label{eq:proximal-linear}
\piprox(x \mid y) \propto
q\!\left(y - \Aapprox \K^{-1} x\right)\,
p\!\left(\K^{-1} x\right).
\end{align}

Accordingly, the Metropolis--Hastings acceptance ratio takes the form
\begin{align}\label{eq:ar-linear}
\min\!\left\{
1,\;
\frac{q(y-\A x^\prime)}{q(y-\Aapprox \xapprox^\prime)}
\frac{q(y-\Aapprox \xapprox_t)}{q(y-\A x_t)}
\frac{p(x^\prime)}{p(\xapprox^\prime)}
\frac{p(\xapprox_t)}{p(x_t)}
\right\},
\end{align}
where $\xapprox_t, \xapprox^\prime \sim \piapprox(x \mid y)$,
$x_t = \K \xapprox_t$ is the current state, and
$x^\prime = \K \xapprox^\prime$ is the proposed state.

The operator $\K$ can be equivalently expressed as
\begin{align}\label{eq:K-linear}
\K = \mathbf{I} + \A^\dagger (\Aapprox-\A),
\end{align}
where $\A^\dagger = (\A^\top \A + \beta \mathbf{I})^{-1}\A^\top$ is a regularized pseudoinverse of $\A$.
%
Under this form of $\K$, the proposal mechanism of \ProxMH
(\cref{eq:proximal-a,eq:proximal-b}) can be interpreted as a correction of an approximate posterior sample $\xapprox \sim \piapprox(x \mid y)$ via
\begin{align}
    x \gets  \xapprox + \A^\dagger (\Aapprox \xapprox - \A \xapprox).
\end{align}
Here, the term $\Aapprox \xapprox - \A \xapprox$ represents a correction of the observation residual induced by the approximate forward model, while the operator $\A^\dagger$ maps this correction back into the parameter (i.e., $x$-) space.

\paragraph{Hyperparameter $\beta$.}
The regularization parameter $\beta$ controls the trade-off between enforcing consistency with the exact forward model and stabilizing the correction around the approximate sample $\xapprox$.
To balance these effects relative to the noise level, we choose $\beta = \Theta(\sigma^2)$, where $\sigma^2$ is the variance of the observational noise.
This scaling prevents overfitting to noise while ensuring that the correction remains effective in mitigating operator mismatch.
In particular, $\beta=\Theta(\sigma^2)$ is required in the KL-divergence analysis (Section~\ref{sec:kl}) and for the mixing-time bounds of \ProxMH (Theorem~\ref{thm:mixtime}). We further investigate the effect of $\beta$ in the numerical experiments.


\subsection{Comparison of Proposal Distributions}\label{sec:kl}
We compare the proposal distributions used in \ApproxMH, \LatentMH, and \ProxMH against \Exact using the expected Kullback--Leibler (KL) divergence, where the expectation is taken with respect to the observation $y$.  For \Approx, we define the expected KL divergence $\KLapprox$ as
\begin{align}\label{eq:kl-define}
    \KLapprox
    :=
    2\,\mathbb{E}_{y}\!\left[
        \KL\!\left(\pi_a(x \mid y)\,\|\,\pi(x \mid y)\right)
    \right].
\end{align}
The quantities $\KLlatent$ and $\KLprox$ are defined analogously for \Latent and \Proximal, respectively.

\begin{assumption}\label{assume:kl}
The prior is $p(x)=\mathcal{N}(\b 0,\b I)$ and the observational noise is
$q(e)=\mathcal{N}(\b 0,\sigma^2 \b I)$.
\end{assumption}
Under \cref{assume:kl}, and assuming that the forward operators $\A$ and $\Aapprox$ are full-rank, all posterior distributions appearing in our analysis are multivariate Gaussian (see \cref{tab:posterior-gaussian}).
\paragraph{Diagonal case.} We first consider a simplified setting in which  both $\F$ and  $\Fapprox$ are diagonal, with an identity observation operator $\Oo=[\Id\ \mathbf{0}]$. Closed-form expressions for the corresponding expected KL divergences are given below.
\begin{theorem}[Adapted from Proposition 3.3 in \cite{chen-biros26}]\label{thm:kl-simple}
Under \cref{assume:kl}, let $\F_{ii}=s_i$ and $\Fapprox_{ii}=\alpha_i s_i$ for $i\in[d]$. Define $\rho_i:=\nicefrac{\alpha_i^2 s_i^2+\sigma^2}{s_i^2+\sigma^2}$ and $\zeta_i:=\nicefrac{1}{(\alpha_i^2 s_i^2+\sigma^2)^2}$.
Setting $\beta=\sigma^2$ in the \texttt{Proximal} proposal, the expected KL divergences are given by
\begin{align}\label{eq:kl-expressions}
    \KLapprox&=-d_y + \sum_{i=1}^{d_y}\!\frac{1}{\rho_i}+\log\rho_i+\zeta_i (\alpha_i - 1)^2(\alpha_i s_i^2 - \sigma^2)^2\frac{s_i^2}{\sigma^2}, \nonumber\\
    \KLlatent&= -d + \sum_{i=1}^{d_y}\!\left(\frac{\alpha_i^2}{\rho_i}+\log\frac{\rho_i}{\alpha_i^2}\right)
+ \sum_{i=d_y+1}^{d}\left(\alpha_i^2+\log\frac{1}{\alpha_i^2}\right)\nonumber\\
  & \quad +\sum_{i = 1}^{d_y}  \zeta_i (\alpha_i^2 - 1)^2 s_i^2 \sigma^2. \nonumber\\
  \KLprox &= -d_y + \sum_{i=1}^{d_y}\!\left(\rho_i+\log\frac{1}{\rho_i}\right)+ \zeta_i (\alpha_i - 1)^2 s_i^2 \sigma^2.
\end{align}
\end{theorem}

Let $\epsilon := |\alpha_i - 1|$ denote the perturbation magnitude. When $\epsilon$ is small, the leading-order behavior is given by
\begin{align}
\KLapprox &\sim \epsilon^2 d_y + \epsilon^2\sum_{i\in[d_y]}\frac{s_i^2}{\sigma^2},\quad\KLlatent \sim \epsilon^2 d + \epsilon^2\sum_{i\in[d_y]}\frac{\sigma^2}{s_i^2},\nonumber\\
\KLprox &\sim \epsilon^2 d_y + \epsilon^2\sum_{i\in[d_y]}\frac{\sigma^2}{s_i^2}.
\end{align}
In most applications of interest, $\sum_i s_i^2/\sigma^2 \gg \sum_i \sigma^2/s_i^2$, which implies that $\KLapprox$ typically dominates both $\KLlatent$ and $\KLprox$. Moreover, while $\KLlatent$ and $\KLprox$ exhibit a comparable mean mismatch, the covariance mismatch of $\KLprox$ is generally smaller, especially when $d\gg d_y$. 


\paragraph{General case.}
Beyond the diagonal case, we study the sensitivity of $\KLapprox$, $\KLlatent$, and $\KLprox$ in a more general setting. We consider $\F=\V \b S \V^\top$ and its approximation as $\Fapprox=\V \Sapprox \V^\top$, where $\V$ is a unitary matrix obtained from the singular value decomposition of a random matrix. The diagonal matrix $\b S$ has entries $\{1/i^2\}_{i\in[d]}$, while $\Sapprox$ has entries $\{\alpha_i/i^2\}_{i\in[d]}$. The observation operator $\b O$ is taken to be a full-rank random matrix.
We evaluate the expected KL divergence by varying four factors: the noise level, the magnitude of the operator perturbation, the observation ratio $(d_y/d)$, and the problem dimension. The experimental design for these sensitivity studies is summarized in \cref{tab:kl-experiments}, and the corresponding results are reported in \cref{fig:kl}. In all regimes tested, $\KLprox$ consistently achieves substantially lower values—often by orders of magnitude—than both $\KLapprox$ and $\KLlatent$.

\begin{figure}[t]
\scriptsize
\centering
\begin{tikzpicture}
\tikzset{
  plottitle/.style={
    rotate=90,
    anchor=center,
    fill=gray!20,
    draw=gray!50,
    rounded corners=2pt,
    inner sep=2pt,
    font=\scriptsize\bfseries,
    align=center
  },
  legendbox/.style={
    fill=white,
    draw=black!30,
    rounded corners=2pt,
    inner sep=1pt
  },
  legendline/.style={
    line width=1.2pt
  }
}
\node[inner sep=0pt] (a)
  {\includegraphics[width=0.2\textwidth]{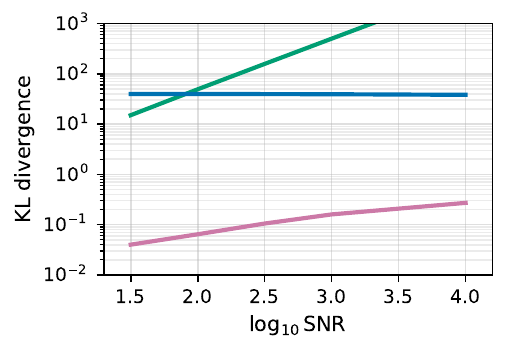}};
\node[plottitle] at ([xshift=-2mm,yshift=2mm]a.west) {Noise level};

\node[inner sep=0pt, right=0.03\textwidth of a] (b)
  {\includegraphics[width=0.2\textwidth]{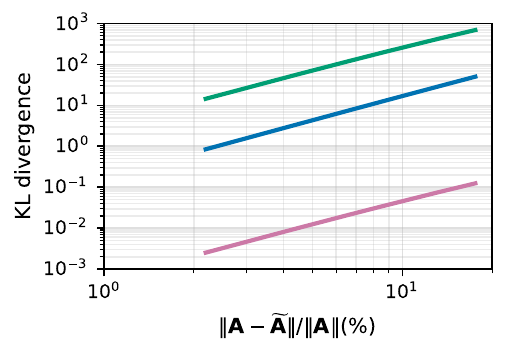}};
  \node[plottitle] at ([xshift=-2mm,yshift=2mm]b.west) {Operator error};
\node[inner sep=0pt, below=0mm of a] (c)
  {\includegraphics[width=0.2\textwidth]{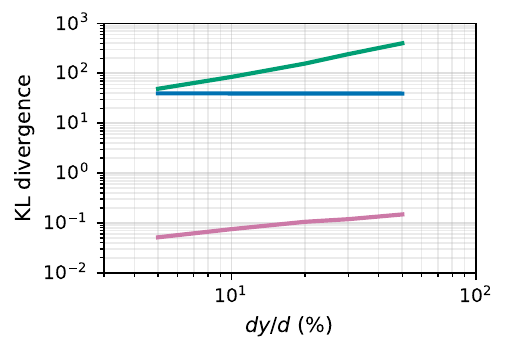}};
  \node[plottitle] at ([xshift=-2mm,yshift=2mm]c.west) {Observation ratio};
\node[inner sep=0pt, below=0mm of b] (d)
  {\includegraphics[width=0.2\textwidth]{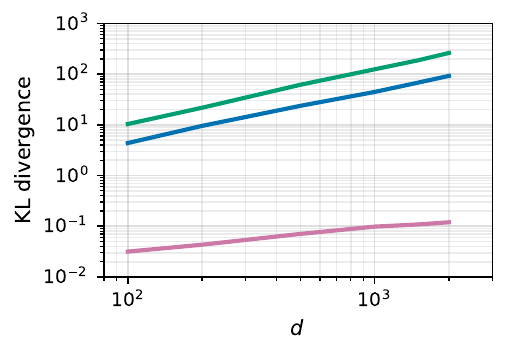}};
    \node[plottitle] at ([xshift=-2mm,yshift=2mm]d.west) {Dimension};
\node[legendbox, anchor=south] at ($(a.north)!0.5!(b.north) + (0,0mm)$) {%
\begin{tikzpicture}[baseline, font=\scriptsize]
  \draw[legendline, color={rgb,255:red,0;green,158;blue,115}] (0,0) -- (5mm,0);
  \node[right=1.3mm] at (5mm,0) {\texttt{Approx} ($\mathbb{D}_a$)};

  \draw[legendline, color={rgb,255:red,0;green,114;blue,178}] (24mm,0) -- (29mm,0);
  \node[right=1.3mm] at (29mm,0) {\texttt{Latent} ($\mathbb{D}_l$)};

  \draw[legendline, color={rgb,255:red,204;green,121;blue,167}] (48mm,0) -- (53mm,0);
  \node[right=1.3mm] at (53mm,0) {\texttt{Proximal} ($\mathbb{D}_p$)};
\end{tikzpicture}%
};
\end{tikzpicture}
\caption{Sensitivity of the expected KL divergence of the \texttt{Approx}, \texttt{Latent}, and \texttt{Proximal} proposals (relative to the \Exact) for different noise levels, operator errors, observation ratios, and dimensions. Details of the experimental setup are given in \cref{tab:kl-experiments}.}
\label{fig:kl}
\end{figure}
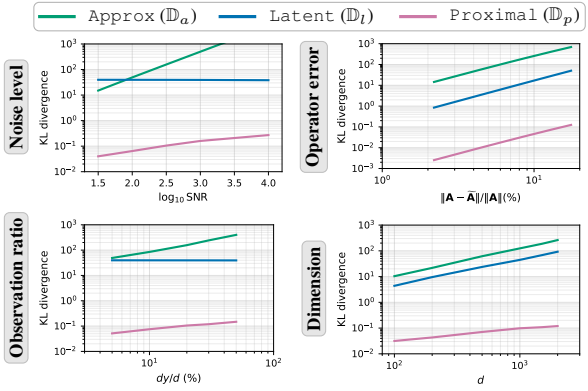

\subsection{Mixing Time of IMH Schemes}\label{sec:mix-time}
We now compare the convergence behavior of different IMH methods in terms of mixing time. The mixing time of the \ApproxMH chain under a warm start $\pi_a(x \mid y)$ is defined as
\begin{align}
\mixapprox(\epsilon)
:= \inf\left\{ n \in \mathbb{N} :
\left\| \pi_a P_a^n - \pi \right\|_{\mathrm{TV}} \le \epsilon \right\},
\end{align}
where $P_a$ denotes the IMH transition kernel with the proposal distribution
$\pi_a(x \mid y)$, and $\|\cdot\|_{\mathrm{TV}}$ denotes the total variation distance.
Analogously, we define the mixing times for \LatentMH and \ProxMH as $\mixlatent(\epsilon)$ and $\mixprox(\epsilon)$, respectively. The theorem below establishes explicit mixing-time bounds for the three IMH schemes under the assumption of a linear forward model and a log-concave prior.
\begin{theorem}[Mixing time for three IMH schemes.]\label{thm:mixtime}
Assume that the prior density $p(x)$ is log-concave and that there exist constants
$L,m>0$ such that $-\log p(x)$ is $L$-smooth and
$m$-strongly convex on $\mathbb{R}^d$. 
Assume further that the noise density $q(e)$ is Gaussian with variance $\sigma^2$, and that the forward operator $\F$ and its approximation $\Fapprox$
are linear and invertible (as required by the \LatentMH proposal). In \ProxMH, set the regularization parameter $\beta = \Theta(\sigma^2)$.
Let $\Delta \A := \Aapprox - \A$.
Then the mixing times of \ApproxMH, \LatentMH, and \ProxMH satisfy
\begin{align}
\mixapprox(\epsilon)
&\;\sim\;
\frac{d}{m^2}\,
\frac{\|\A^\top \Delta \A\|^2}{\sigma^4}\,
\log(1/\epsilon),\label{eq:mix-approx}\\
\mixlatent(\epsilon)
&\;\sim\;
\frac{d L^2}{m^2}\,
\left\| \mathbf{I} - \Fapprox^{-1} \F \right\|^2
\log(1/\epsilon),\label{eq:mix-latent}\\
\mixprox(\epsilon)
&\;\sim\;
\frac{d L^2}{m^2}\,
\left\| \mathbf{I} - \K^{-1} \right\|^2
\log(1/\epsilon).\label{eq:mix-prox}
\end{align}
\end{theorem}
\begin{proof}[Proof sketch]
The result follows from the general mixing-time bound in \cref{thm:mix-general},
which reduces the analysis to bounding local Lipschitz constants of the log-weight
functions associated with each IMH proposal. The detailed proof is provided in \cref{sec:proof-mix-time}.
For \ApproxMH, the log-weight admits a quadratic form involving the operator mismatch $\Delta \A$, resulting in a Lipschitz constant proportional to $\|\A^\top \Delta \A\|/\sigma^2$. 
For \LatentMH, smoothness of the negative log-prior implies that the Lipschitz
constant is controlled by $\|\mathbf{I}-\Fapprox^{-1}\F\|$.
For \ProxMH, the correction operator $\K^{-1}$ introduces additional regularization,
and choosing $\beta=\Theta(\sigma^2)$ ensures dependence on $\|\mathbf{I}-\K^{-1}\|$.
Substituting these bounds into \cref{thm:mix-general} yields the stated scaling of
the mixing times.
\end{proof}
\vspace{-0.6em}

As shown in \cref{thm:mixtime}, the mixing time of \ApproxMH scales unfavorably with the signal-to-noise ratio $\|\A\|/\sigma^2$, which is typically large in inverse problems. Both \LatentMH and \ProxMH remove this explicit dependence, but differ in how operator
errors are amplified.

For \LatentMH, the term $\|\mathbf{I}-\Fapprox^{-1}\F\|$ depends on the full perturbation
$\Delta\F=\Fapprox-\F$ and is amplified by the conditioning of $\Fapprox$ through
$\|\Fapprox^{-1}\|$, which can become large even for small $\|\Delta\F\|$. In contrast, \ProxMH depends only on the observable perturbation $\Delta\A=\mathbf{O}\Delta\F$, filtered through the regularized pseudoinverse $\A^\dagger$ in $\K=\mathbf{I}+\A^\dagger(\Aapprox-\A)$. As a result, the mixing behavior of \ProxMH is more robust to operator perturbations.

A sharper comparison can be obtained in the diagonal setting introduced in the previous section for the KL-divergence comparison. In that case, one can show that
\begin{align}
\|\mathbf{I}-\K^{-1}\|\leq
\|\mathbf{I}-\Fapprox^{-1}\F\|,
\end{align}
with equality attained only when  $\F=\Fapprox$.


\subsection{Nonlinear Operator}\label{sec:gn}
We now extend \ProxMH to nonlinear forward problems. Recall the optimization problem \cref{eq:proximal-b}: given a sample $\xapprox \sim \piapprox(x \mid y)$ from the approximate posterior, it is expected that the proposal $x$ will remain close to $\xapprox$ while reducing the data misfit.
In the nonlinear least-squares setting, Gauss--Newton methods are known to achieve local quadratic convergence when the residual is small, making them a natural choice for this correction step.

Specifically, we consider a single Gauss--Newton step applied to \cref{eq:proximal-b},
\begin{align}\label{eq:gn}
x \gets GN(\xapprox)
:= \xapprox -
\left(\J(\xapprox)^\top \J(\xapprox) + \beta \mathbf{I}\right)^{-1}
\J(\xapprox)^\top r(\xapprox),
\end{align}
where $r(\xapprox) = \A(\xapprox) - \Aapprox(\xapprox)$ denotes the residual and $\J(\xapprox)$ is the Jacobian of $\A(\cdot)$ evaluated at $\xapprox$. When $\A$ is linear, a single Gauss--Newton step recovers the exact solution given in \cref{eq:sol-linear}.

Under this construction, the \ProxMH proposal induced by one
Gauss--Newton step can be expressed as a pushforward distribution,
\begin{align}\label{eq:proximal-nonlinear}
\piprox(x \mid y) = \piapprox(\xapprox \mid y)\,|\det \J_{GN}(\xapprox)|^{-1},
\end{align}
where $\xapprox = GN^{-1}(x)$ and $\J_{GN}(\xapprox)$ denotes the Jacobian of the Gauss--Newton map $GN(\cdot)$ at $\xapprox$. The corresponding Metropolis--Hastings acceptance ratio is
\begin{align}\label{eq:ar-nonlinear}
\min\!\left\{ 1,\;
\frac{\pi(x^\prime \mid y)}{\pi(x_t \mid y)}
\frac{\piapprox(\xapprox_t \mid y)}{\piapprox(\xapprox^\prime \mid y)}
{\color{tabblue}\frac{|\det \J_{GN}(\xapprox^\prime)|}
     {|\det \J_{GN}(\xapprox_t)|}}
\right\},
\end{align}
where $\xapprox_t, \xapprox^\prime \sim \piapprox(x \mid y)$ and $x_t = GN(\xapprox_t)$ denote the current state, $x^\prime = GN(\xapprox^\prime)$ is the proposed state.

Compared with the linear acceptance ratio in \cref{eq:ar-linear}, the only additional term is the Jacobian determinant ratio. Since the Gauss--Newton update \cref{eq:gn} represents a small, local correction around $\xapprox$, we expect this determinant ratio to be close to one in practice and thus to have a negligible impact on the acceptance ratio.

To illustrate this setting concretely, we consider a nonlinear model $y = \A G(x) + e$, where $\A$ is a full-rank matrix and
$G$ is a nonlinear mapping. We assume a Gaussian prior on $x$ and additive Gaussian noise $e \sim \mathcal{N}(0,\sigma^2 \mathbf{I})$.  This model is standard in inverse problems with deep generative model--informed priors, including generative adversarial networks (GANs), variational autoencoders (VAEs), and normalizing flows, where $x$ represents the latent variable and $G$ corresponds to the decoder or generator implemented by a neural network. 

Consider a small perturbation $\Delta \A = \Aapprox - \A= \delta \A$ for some $\delta \in (0,1)$. For arbitrary points $\xapprox_1$ and $\xapprox_2$, we have
\begin{align}\label{eq:log-det-ratio}
\log
\frac{|\det \J_{GN}(\xapprox_1)|}
     {|\det \J_{GN}(\xapprox_2)|}
=
\sum_{i\in[d_x]}
\log
\frac{1 - \delta \lambda_{1,i}/(\lambda_{1,i} + \beta)}
     {1 - \delta \lambda_{2,i}/(\lambda_{2,i} + \beta)},
\end{align}
where $\{\lambda_{1,i}\}$ and $\{\lambda_{2,i}\}$
denote the eigenvalues (in descending order) of
$\J^\top(\xapprox_1)\J(\xapprox_1)$ and
$\J^\top(\xapprox_2)\J(\xapprox_2)$, respectively. Since $\beta = \Theta(\sigma^2)$, when
$\lambda_{1,i}, \lambda_{2,i} \gg \beta$ the corresponding terms are negligible.
When $\lambda_{1,i}$ and $\lambda_{2,i}$ are comparable to $\beta$, the log ratio is determined by their difference but remains uniformly bounded by $\log(1-\delta)^{-1}$.

In our nonlinear implementation, we therefore neglect the Jacobian determinant ratio when computing the acceptance probability. We further examine the validity of this approximation in the numerical experiments. For specific applications, however, we recommend monitoring the scale of
this ratio to determine whether its contribution can be safely ignored.

\section{Numerical Experiments}\label{sec:results}

We design numerical experiments to compare \ProxMH with existing IMH schemes, to evaluate its nonlinear extension, and to assess consistency with theory.
Performance is measured using acceptance rates, relative mean error, and relative componentwise second-moment error. In all experiments, the observational noise is Gaussian and the noise-to-signal ratio $\|e\|/\|y\|$ is controlled to be between 15\% and 20\%. 
We additionally include the No-U-Turn Sampler (NUTS) as a baseline to provide a reference for sampling quality and computational efficiency. Because NUTS requires repeated evaluations of the exact forward operator, it is used primarily as a benchmark and is not considered a directly comparable low-cost alternative. 
We present a linear forward model test in \cref{sec:2well} and two nonlinear forward model tests in \cref{sec:mnist,sec:test-nonlinear}.


\begin{figure*}[t]
\scriptsize
\centering
\begin{tikzpicture}

\def\w{0.20\textwidth}        
\def\xgap{0\textwidth}     
\def\ygap{0.14\textwidth}     

\def\xA{0}
\def\xB{\w+\xgap}
\def\xC{2*\w+2*\xgap}
\def\xD{3*\w+3*\xgap}

\def\yone{0}
\def\ytwo{-\ygap}
\def\ythree{-2*\ygap}

\node[anchor=west] at (-3, \yone) {\textbf{Test I}};
\node[anchor=west] at (-3, \ytwo) {\textbf{Test II}};
\node[anchor=west] at (-3, \ythree) {\textbf{Test III}};

\node[inner sep=0pt] (a11) at (\xA,\yone) {\includegraphics[width=\w]{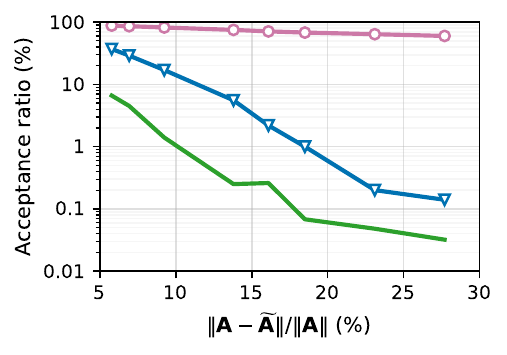}};
\node[inner sep=0pt] (a12) at (\xB,\yone) {\includegraphics[width=\w]{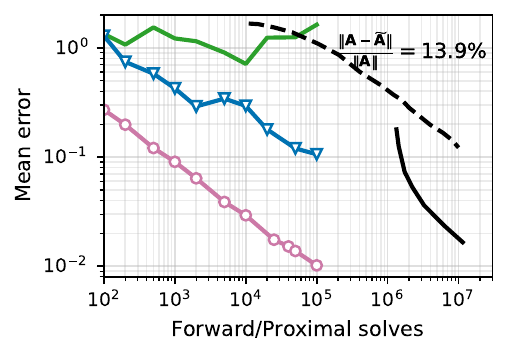}};
\node[inner sep=0pt] (a13) at (\xC,\yone) {\includegraphics[width=\w]{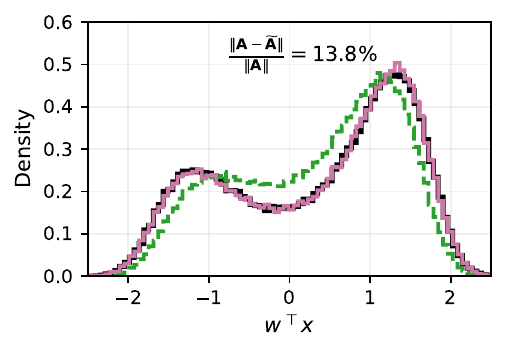}};
\node[inner sep=0pt] (a14) at (\xD,\yone) {\includegraphics[width=\w]{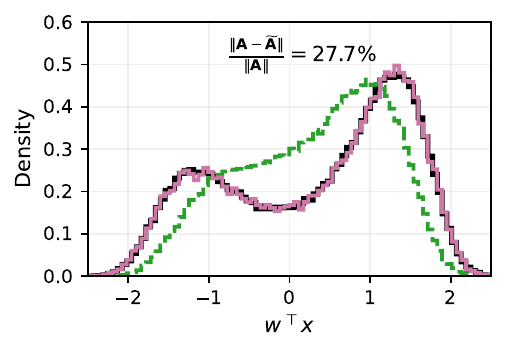}};

\node[inner sep=0pt] (a21) at (\xA,\ytwo) {\includegraphics[width=\w]{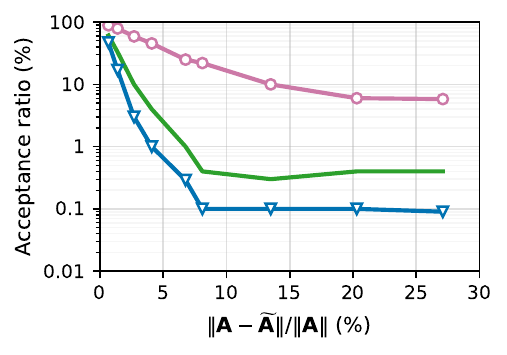}};
\node[inner sep=0pt] (a22) at (\xB,\ytwo) {\includegraphics[width=\w]{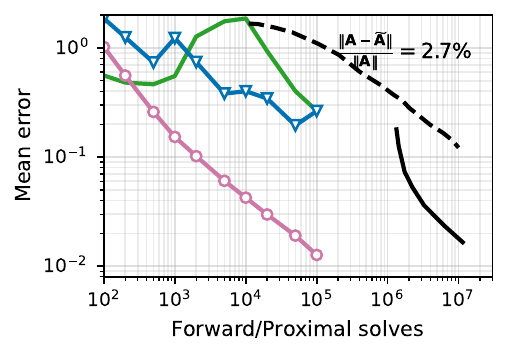}};
\node[inner sep=0pt] (a23) at (\xC,\ytwo) {\includegraphics[width=\w]{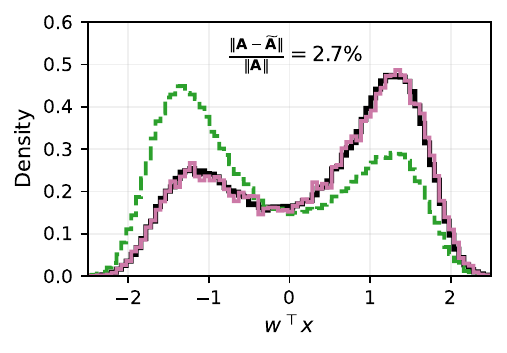}};
\node[inner sep=0pt] (a24) at (\xD,\ytwo) {\includegraphics[width=\w]{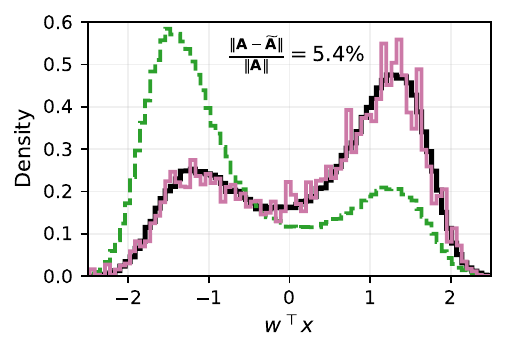}};

\node[inner sep=0pt] (a31) at (\xA,\ythree) {\includegraphics[width=\w]{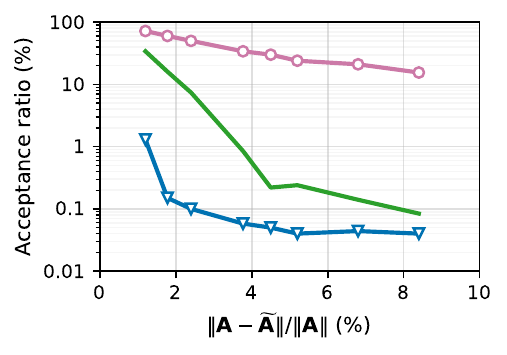}};
\node[inner sep=0pt] (a32) at (\xB,\ythree) {\includegraphics[width=\w]{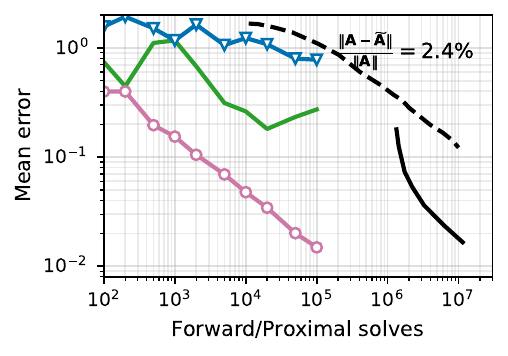}};
\node[inner sep=0pt] (a33) at (\xC,\ythree) {\includegraphics[width=\w]{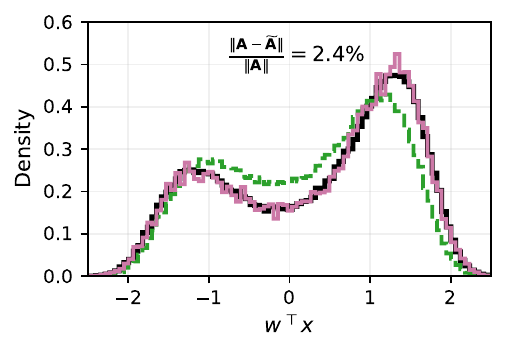}};
\node[inner sep=0pt] (a34) at (\xD,\ythree) {\includegraphics[width=\w]{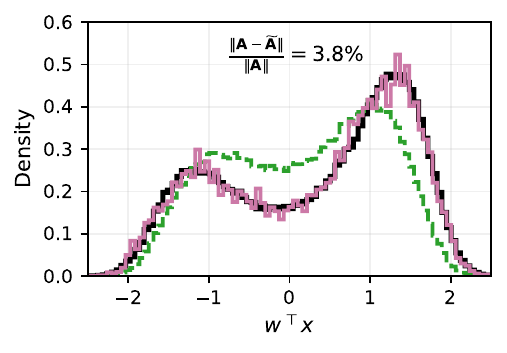}};

\node[inner sep=0pt, anchor=south west]
  at ([xshift = 5mm, yshift=-1mm] a11.north west)
  {\includegraphics[width=0.35\textwidth]{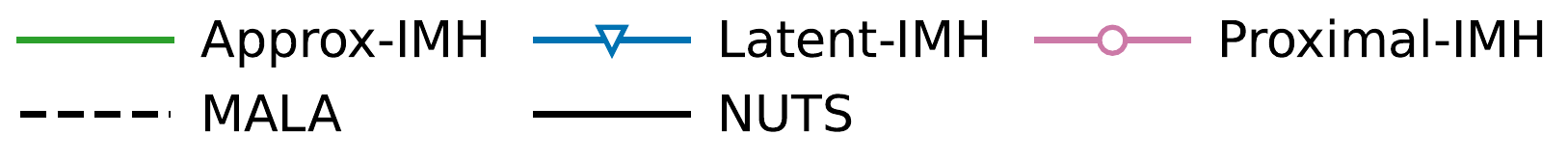}};

\node[inner sep=0pt, anchor=south east]
  at ([yshift=1mm] a14.north east)
  {\includegraphics[width=0.38\textwidth]{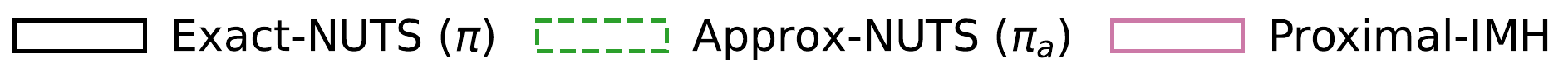}};
\end{tikzpicture}
\caption{Comparison of sampling performance in the bimodal test.
Rows correspond to Tests~I--III. From left to right, panels show Metropolis--Hastings acceptance rates for three IMH methods, convergence of the relative mean error for the IMH methods together
with NUTS and MALA, and histograms of the projected samples $w^\top x$ for two levels of operator error.
In the histogram panels, Approx-NUTS denotes samples drawn from the \Approx $\pi_a(x\mid y)$.
For \ProxMH, we set $\beta=\sigma^2$. Results in the left panel are averaged over 5 independent trials for each
method.}
\label{fig:2well-converge}
\end{figure*}
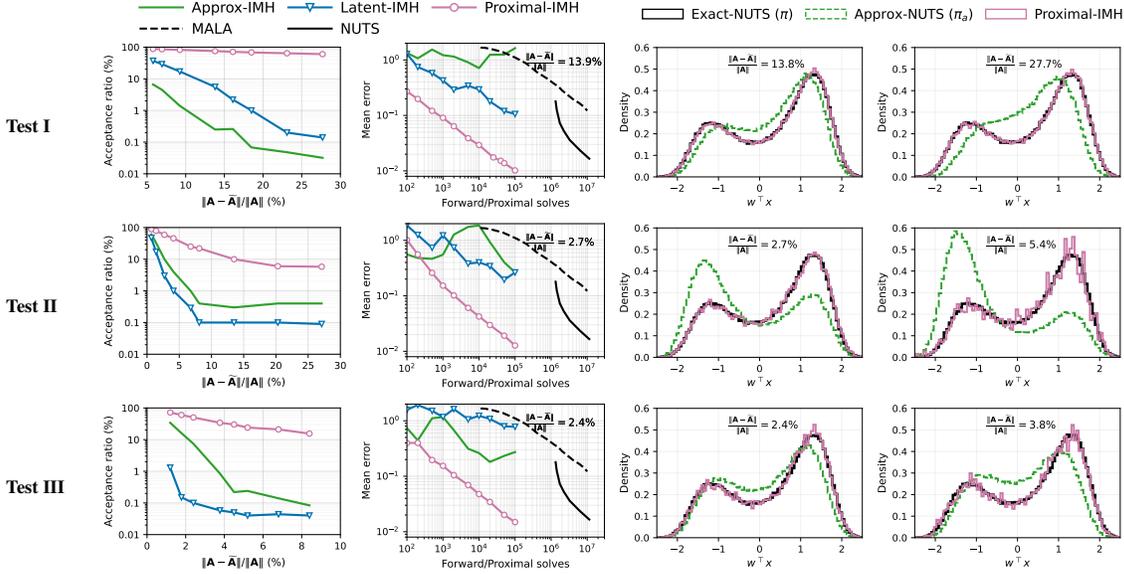

\subsection{Bimodal Test.}\label{sec:2well}
We consider the following prior distribution
\begin{align}\label{eq:prior-bimodal}
    p(x) \propto \exp\!\left(
    -\frac{1}{2}\|x\|^2
    - \tau (w^\top x - c)^2 (w^\top x + c)^2
    \right),
\end{align}
where $x \in \mathbb{R}^{200}$, $w \in \mathbb{R}^{200}$ satisfies $\|w\|=1$, and we set $c=2$ and $\tau=0.3$. This prior is bimodal along the projection direction $w^\top x$. We visualize the marginal prior distribution of $w^\top x$ in \cref{fig:2well-prior}.

We construct $\A=\mathbf{O}\mathbf{F}$ with
$\mathbf{F}=\mathbf{V}\mathbf{S}\mathbf{V}^\top$, where $\mathbf{V}$ is orthogonal, $\b S$ is diagonal with $\mathbf{S}_{ii}=1/i$, $d_y=50$, and $\mathbf{O}\in\mathbb{R}^{d_y\times d_x}$ is
well conditioned.

We consider three different constructions of the approximate forward operator $\Fapprox$, with $\Aapprox = \b O \Fapprox$:
\begin{itemize}[leftmargin=*,itemsep=0pt, topsep=2pt]
\item \textbf{Test I (Multiplicative spectral perturbation).}
 $\Fapprox = \b V \Sapprox \b V^\top$, where $\Sapprox$ is diagonal with entries $\Sapprox_{ii} = \alpha_i \b S_{ii}$, $\alpha_i \sim \mathcal U[\alpha_{-}, \alpha_{+}]$,  with $\alpha_{-} \in (0,1)$ and $\alpha_{+} > 1$. 
\item \textbf{Test II (Additive low-rank perturbation).}
We define $\Fapprox = \F + \epsilon \b U_1 \b U_2^\top$,
where $\b U_1, \b U_2 \in \mathbb{R}^{d_x \times 5}$ have i.i.d.\ standard normal entries.
\item \textbf{Test III (Low-rank approximation).}
We define $\Fapprox = \b V \Sapprox \b V^\top$, where $\Sapprox$ is obtained by setting $\Sapprox_{ii} = \b S_{ii}$ if $\b S_{ii} > \alpha$ and
$\Sapprox_{ii} = 0$ otherwise, for a threshold $\alpha>0$.
\end{itemize}

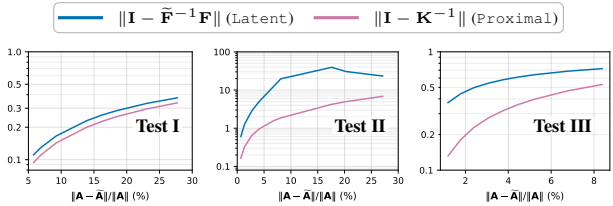
\begin{figure}[t]
\centering
\begin{tikzpicture}
\tikzset{
  legendbox/.style={
    fill=white,
    draw=black!30,
    rounded corners=2pt,
    inner sep=1pt
  },
  legendline/.style={
    line width=1.2pt
  },
  testlabel/.style={
    font=\scriptsize\bfseries,
    anchor=south east,
    fill=white,
    fill opacity=0.75,
    text opacity=1,
    inner sep=1pt,
    rounded corners=1pt
  }
}

\node[inner sep=0pt] (a)
  {\includegraphics[width=0.33\columnwidth]{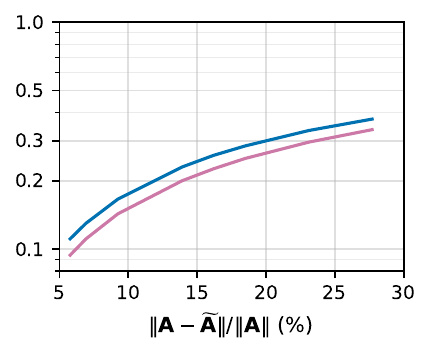}};
\node[inner sep=0pt, right=0pt of a] (b)
  {\includegraphics[width=0.33\columnwidth]{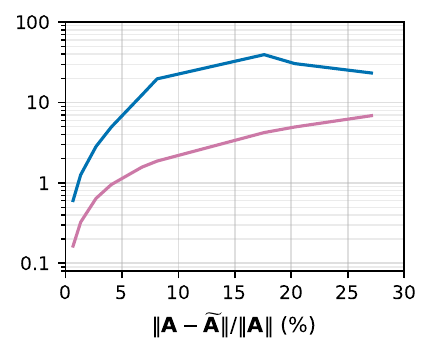}};
\node[inner sep=0pt, right=0pt of b] (c)
  {\includegraphics[width=0.33\columnwidth]{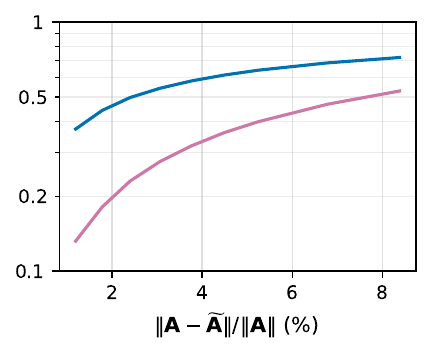}};

\node[testlabel] at ($(a.south east)+(-3mm,10mm)$) {Test I};
\node[testlabel] at ($(b.south east)+(-3mm,10mm)$) {Test II};
\node[testlabel] at ($(c.south east)+(-3mm,10mm)$) {Test III};

\node[legendbox, anchor=south]
  at ($(a.north)!0.5!(c.north) + (0,1mm)$) {%
\begin{tikzpicture}[baseline, font=\tiny]
  \draw[legendline, color={rgb,255:red,0;green,114;blue,178}]
    (0,0) -- (6mm,0);
  \node[right=1.5mm] at (6mm,0)
    {$\|\b I-\Fapprox^{-1}\F\|$ (\texttt{Latent})};
  \draw[legendline, color={rgb,255:red,204;green,121;blue,167}]
    (34mm,0) -- (40mm,0);
  \node[right=1.5mm] at (40mm,0)
    {$\|\b I-\K^{-1}\|$ (\texttt{Proximal})} ;
\end{tikzpicture}%
};
\end{tikzpicture}
\caption{Operator discrepancy in the bimodal test.
The quantities $\|\mathbf I-\Fapprox^{-1}\F\|$ and $\|\mathbf I-\K^{-1}\|$ are reported for \LatentMH and \ProxMH, respectively, where $\K$ is defined in \cref{eq:K-linear}. For \ProxMH, the hyperparameter is set to $\beta=\sigma^2$.}
\label{fig:2well-norm}
\end{figure}


Figure~\ref{fig:2well-converge} compares the sampling performance of different methods in the bimodal test between Tests~I--III. Across Tests~I--III, \ProxMH consistently outperforms \ApproxMH and \LatentMH,
achieving substantially higher acceptance rates, particularly as $\|\A-\Aapprox\|/\|\A\|$ increases.

In Test~II, the additive low-rank perturbation $\Delta\F$ significantly degrades the conditioning of $\F$, leading to amplified inverse-operator errors. To interpret this behavior, we report in \cref{fig:2well-norm} the operator discrepancy terms that appear in the mixing-time bounds \cref{eq:mix-latent,eq:mix-prox}. Although the theoretical analysis in \cref{sec:mix-time} assumes log-concave targets and the bimodal prior used here violates this assumption, the same operator-norm quantities still govern proposal sensitivity. In particular, the filtered error $\|\b I-\K^{-1}\|$ remains stable under perturbations that cause $\|\b I-\Fapprox^{-1}\F\|$ to grow rapidly, which explains the improved robustness of \ProxMH relative to \LatentMH.

From the convergence plots in \cref{fig:2well-converge}, \ProxMH converges significantly faster than the other IMH variants and achieves accuracy comparable to NUTS while using independent proposals. MALA converges slowly in this setting due to both the bimodal posterior and the large smoothness constant of the negative log-posterior, which scales as $L+\|\A\|^2/\sigma^2$ and leads to an unfavorable mixing-time dependence \cite{dwivedi2019}.

The right columns show histograms of the projected samples $w^\top x$ for two representative perturbation levels.
Approx-NUTS, which samples the approximate posterior $\pi_a(x\mid y)$, exhibits biased mode weights, whereas \ProxMH accurately captures both modes in all tests.

Finally, \cref{fig:2well-beta} presents the effect of the hyperparameter $\beta$ in \ProxMH. In all tests, higher acceptance rates correlate with lower relative mean errors, with optimal performance consistently achieved near $\beta=\sigma^2$. Performance degrades noticeably for $\beta>2\sigma^2$, indicating increased sensitivity to over-regularization.

\begin{figure}[t]
\centering
\begin{tikzpicture}

\tikzset{
  panel/.style={inner sep=0pt},
  paneltitle/.style={
    font=\scriptsize,
    anchor=north,
    yshift=2mm
  },
  ylabelL/.style={
    font=\tiny,
    rotate=90,
    text={rgb,255:red,0;green,114;blue,178}
  },
  ylabelR/.style={
    font=\tiny,
    rotate=90,
    text={rgb,255:red,204;green,121;blue,167}
  }
}

\node[panel] (a)
  {\includegraphics[width=0.32\columnwidth]{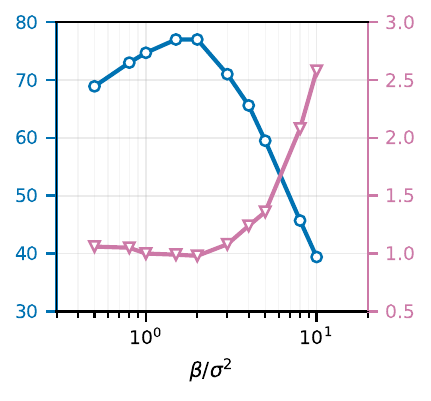}};
\node[panel, right=0mm of a] (b)
  {\includegraphics[width=0.32\columnwidth]{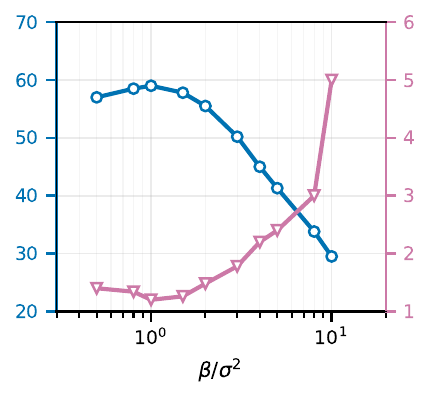}};
\node[panel, right=0mm of b] (c)
  {\includegraphics[width=0.32\columnwidth]{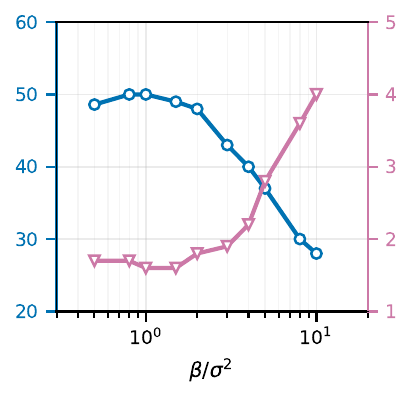}};

\node[paneltitle, below=1.5mm of a.south] {\textbf{Test I} (13.9\%)};
\node[paneltitle, below=1.5mm of b.south] {\textbf{Test II} (2.7\%)};
\node[paneltitle, below=1.5mm of c.south] {\textbf{Test III} (2.4\%)};

\node[ylabelL, anchor=south, xshift=2mm, yshift=-1.5mm]
  at (a.west) {Acceptance ratio (\%)};

\node[ylabelR, anchor=south, xshift=2mm, yshift=-3mm]
  at (c.east) {Relative mean error (\%)};

\end{tikzpicture}
\caption{Effect of $\beta$ on \ProxMH in the bimodal test.
Values in parentheses represent $\|\A-\Aapprox\|/\|\A\|$. Blue curves show acceptance rates (left y-axis), and pink curves show relative mean errors after $10^5$ MH steps (right y-axis).}
\label{fig:2well-beta}
\end{figure}
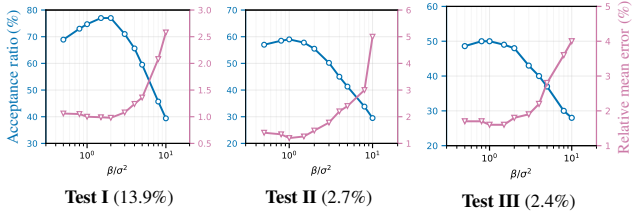

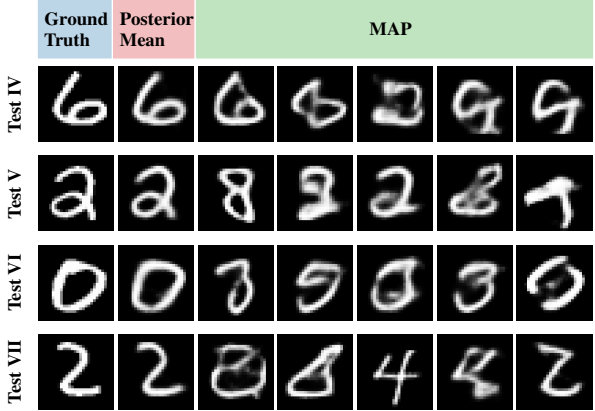
\begin{figure}[t]
\centering
\scriptsize
\begin{tikzpicture}
\draw[fill=tabgreen!30, draw=white, line width=.5pt] (-1.6,0.6)  rectangle (3.7,1.4);
\draw[fill=tabred!30, draw=white, line width=.5pt] (-2.7,0.6)  rectangle (-1.6,1.4);
\draw[fill=tabblue!30, draw=white, line width=.5pt] (-3.7,0.6)  rectangle (-2.7,1.4);

\node[inner sep=0pt] (a)
  {\includegraphics[width=0.9\columnwidth]{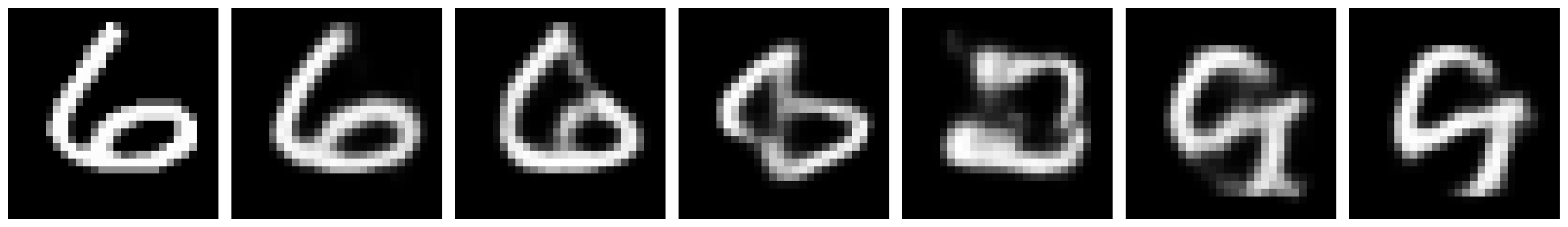}};
\node[rotate=90, anchor=south, font=\bfseries]
  at ([xshift=-1mm]a.west) {Test~IV};

\node[inner sep=0pt, below=1mm of a] (b)
  {\includegraphics[width=0.9\columnwidth]{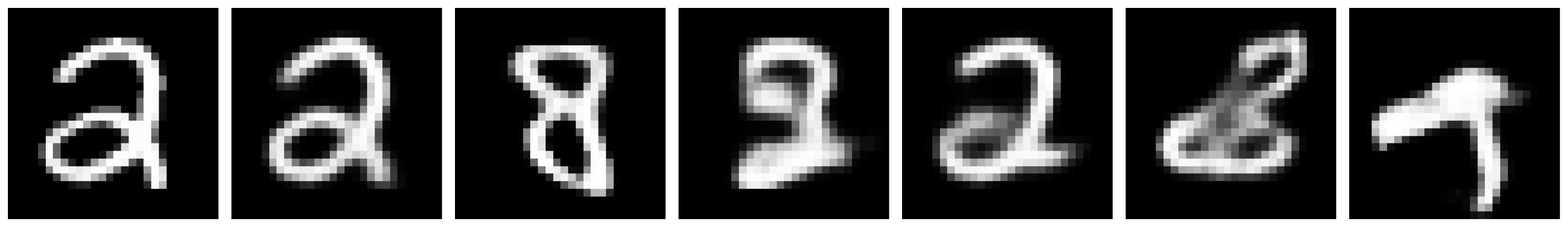}};
\node[rotate=90, anchor=south, font=\bfseries]
  at ([xshift=-1mm]b.west) {Test~V};
  
\node[inner sep=0pt, below=1mm of b] (c)
  {\includegraphics[width=0.9\columnwidth]{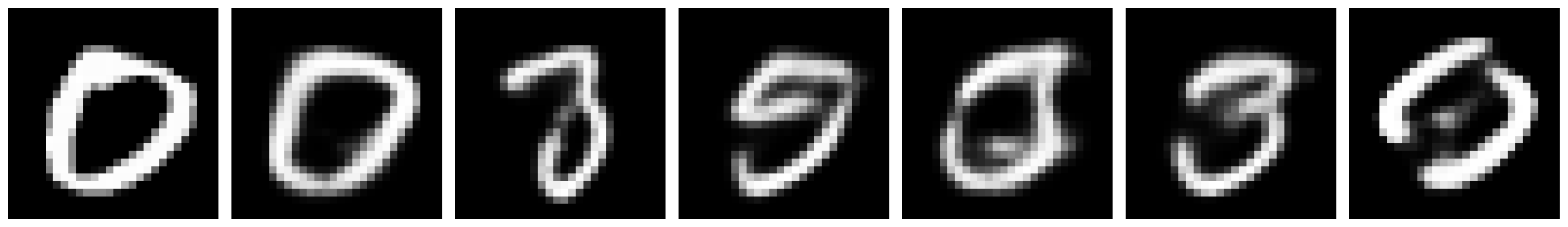}};
\node[rotate=90, anchor=south, font=\bfseries]
  at ([xshift=-1mm]c.west) {Test~VI};
  
\node[inner sep=0pt, below=1mm of c] (d)
  {\includegraphics[width=0.9\columnwidth]{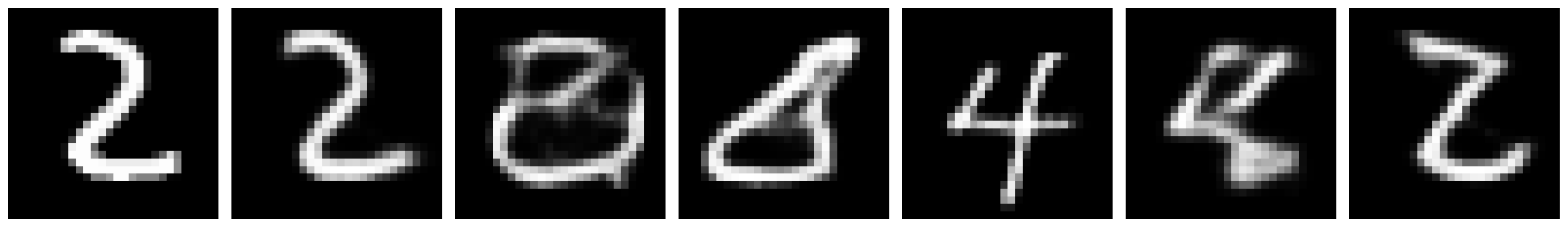}};
\node[rotate=90, anchor=south, font=\bfseries]
  at ([xshift=-1mm]d.west) {Test~VII};
\node[inner sep=0pt] (3) at (1,1) {\textbf{MAP}};
\node[anchor=center,inner sep=0pt, text width =10mm] (3) at (-3.1,1) {\textbf{Ground Truth}};
\node[inner sep=0pt, text width =10mm] (3) at (-2.1,1) {\textbf{Posterior Mean}};
\end{tikzpicture}
\caption{Ground truth, posterior mean and MAP estimates for the MNIST inverse problem in~\textbf{Tests~IV--VII}. The MAP estimates are obtained from different random initial guesses for the optimization problem. We observe that in several cases the reconstructed MAP point is wrong. This is due to the nonconvexity of solving for a MAP point. }
\label{fig:mnist-compare}
\end{figure}


\begin{figure}[t]
\centering
\scriptsize
\newlength{\rowgap}
\newlength{\leggap}
\newlength{\colgap}
\setlength{\rowgap}{2mm}
\setlength{\leggap}{6mm}
\setlength{\colgap}{3mm}
\begin{tikzpicture}[
  rowlabel/.style={
    rotate=90,
    anchor=center,
    text width=10mm,
    align=center,
    font=\bfseries,
    inner sep=1pt
  },
  legendbox/.style={
    anchor=south,
    draw=gray!70,
    line width=0.5pt,
    rounded corners=2pt,
    inner sep=2pt
  }
]

\node[inner sep=0pt, xshift=10mm] (r1c1)
  {\includegraphics[width=0.44\columnwidth]{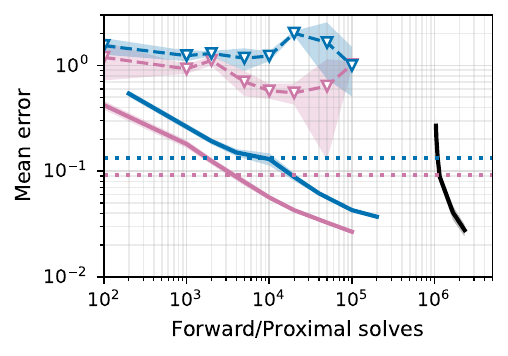}};
\node[inner sep=0pt, right=3mm of r1c1] (r1c2)
  {\includegraphics[width=0.44\columnwidth]{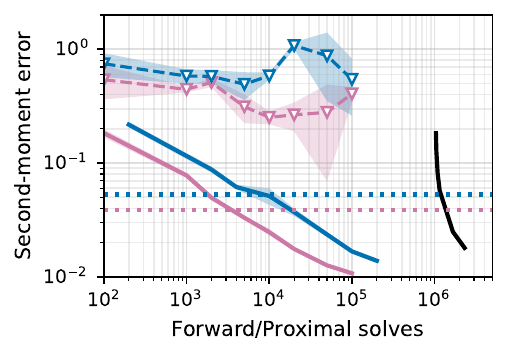}};
\node[rowlabel] at ([xshift=-2mm]r1c1.west) {Test IV};

\node[legendbox] at ([xshift=23mm, yshift=-1mm]r1c1.north)
  {\includegraphics[width=0.8\columnwidth]{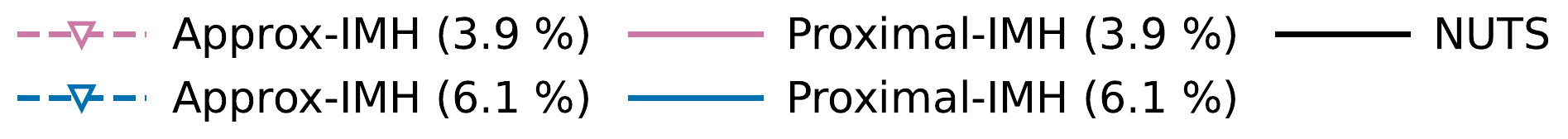}};

\node[inner sep=0pt, below=\rowgap of r1c1] (r2c1)
  {\includegraphics[width=0.44\columnwidth]{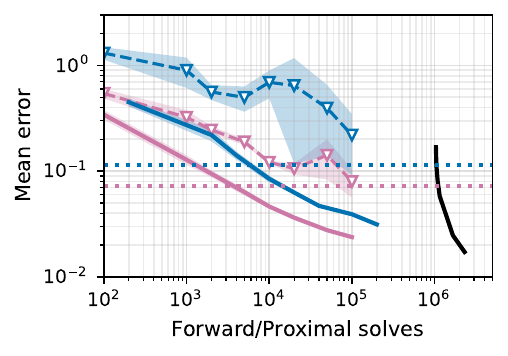}};
\node[inner sep=0pt, right=\colgap of r2c1] (r2c2)
  {\includegraphics[width=0.44\columnwidth]{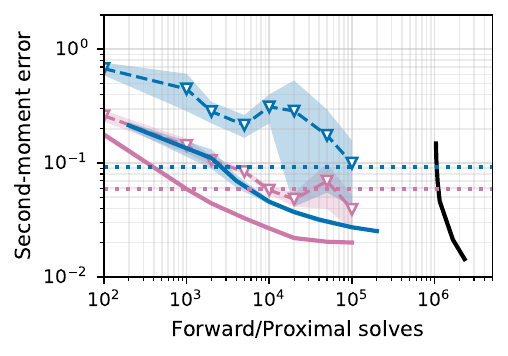}};
\node[rowlabel] at ([xshift=-2mm]r2c1.west) {Test V};

\node[inner sep=0pt, below=\leggap of r2c1] (r3c1)
  {\includegraphics[width=0.44\columnwidth]{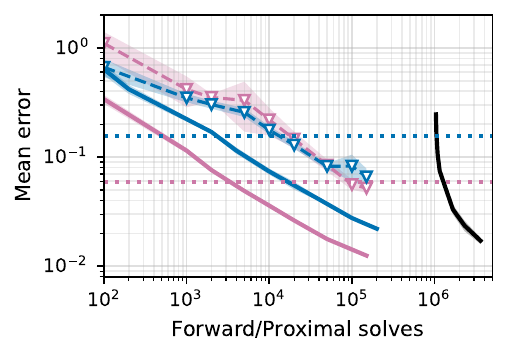}};
\node[inner sep=0pt, right=\colgap of r3c1] (r3c2)
  {\includegraphics[width=0.44\columnwidth]{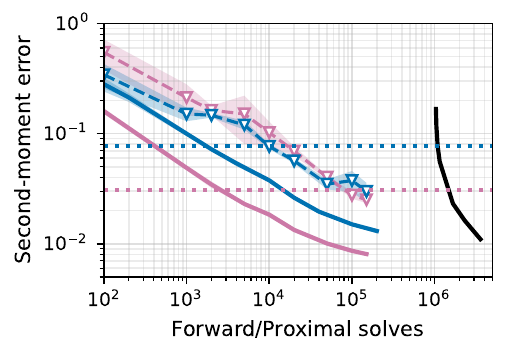}};
\node[rowlabel] at ([xshift=-2mm]r3c1.west) {Test VI};

\node[legendbox] at ([xshift=23mm, yshift=-1mm]r3c1.north)
  {\includegraphics[width=0.8\columnwidth]{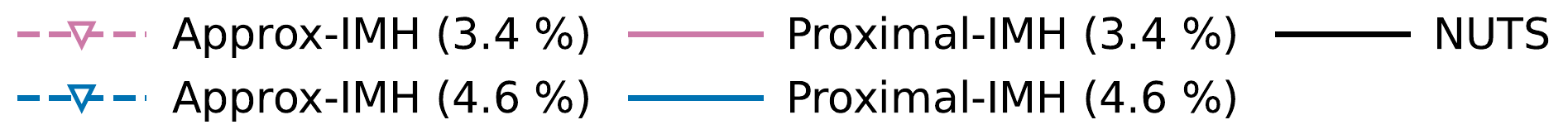}};

\node[inner sep=0pt, below=\rowgap of r3c1] (r4c1)
  {\includegraphics[width=0.44\columnwidth]{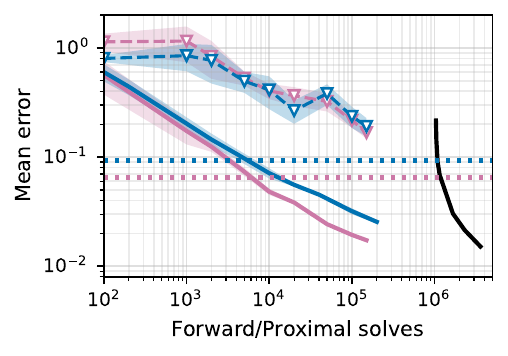}};
\node[inner sep=0pt, right=\colgap of r4c1] (r4c2)
  {\includegraphics[width=0.44\columnwidth]{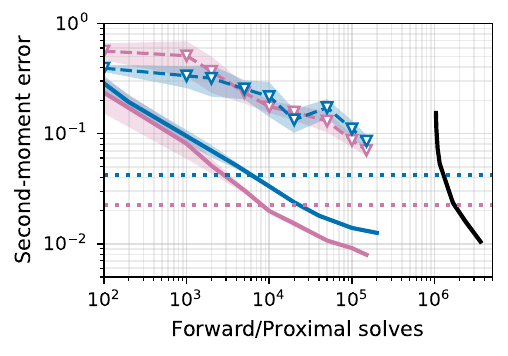}};
\node[rowlabel] at ([xshift=-2mm]r4c1.west) {Test~VII};
\end{tikzpicture}
\caption{Sampling performance on MNIST imaging inverse problems: convergence of the relative mean error and relative second-moment error for different sampling methods. Dashed lines indicate the bias induced by the \Approx $\pi_a(x\mid y)$. For \ProxMH, we set $\beta=\sigma^2$. Results in the convergence plots are reported over 5 independent trials.}
\label{fig:mnist-converge}
\end{figure}
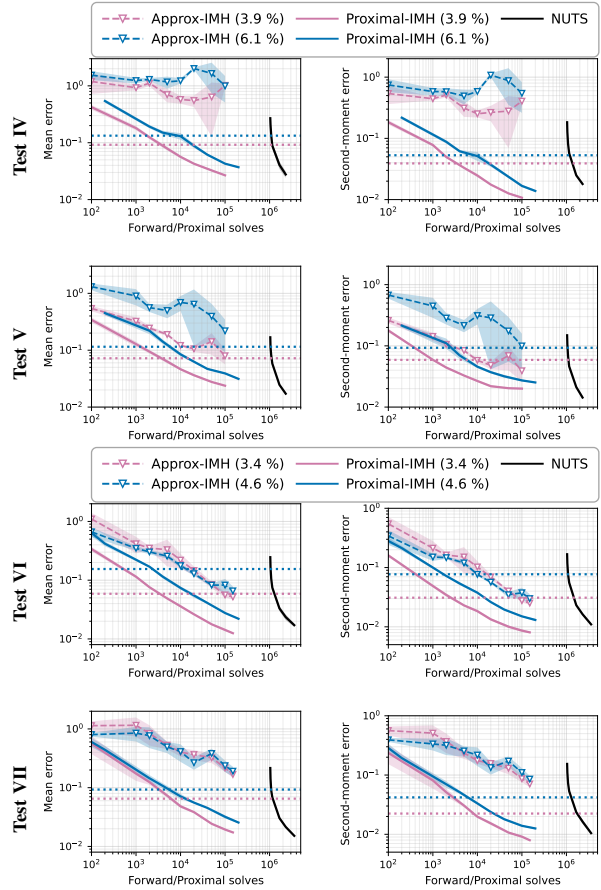

\subsection{MNIST test.}\label{sec:mnist}
We consider an inverse imaging problem with the forward model $y = \A x + e$, where $x$ denotes the image in pixel space, $\A$ is a full-rank linear operator, and $y$ is the observed signal.
As a prior, we use the decoder of a variational autoencoder (VAE) trained on the MNIST dataset~\cite{lecun1998gradient}.
Specifically, let $G(\cdot)$ denote the decoder mapping latent variables to images in the pixel space. The forward model in latent space is then given by $y = \A G(z) + e$. In our experiments, we set $d_z = 128$, $d_x = 784$, and $d_y = 100$. We focus on sampling the posterior distribution in the latent space:
\begin{align}\label{eq:mnist-post}
    \pi(z \mid y) \propto
    \exp\!\big(
    -\frac{1}{2\sigma^2} \|\A G(z) - y\|^2
    - \frac{1}{2}\|z\|^2
    \big).
\end{align}

We randomly select digits from the MNIST test dataset, which were not seen during VAE training, and generate observations through $y = \A x + e$. We consider two types of exact operators $\A$ and approximations $\Aapprox$:
\begin{itemize}[leftmargin=*, itemsep=0pt, topsep=2pt]
    \item \textbf{Tests~IV and V}: The operators $\A$ and $\Aapprox$ are constructed in the same manner as in Test~I. The resulting operator $\A$ has condition number $\kappa(\A)\approx 60$.
    \item \textbf{Tests~VI and VII (linearized Helmholtz operator)}: The operator $\A$ is constructed using a linearized Helmholtz operator based on the Born Approximation~\cite{colton-kress98} on a fine grid ($100\times100$). The approximation $\Aapprox$ is constructed on coarser grids ($35\times35$ or $20\times20$). The operator $\A$ has condition number $\kappa(\A) \approx 420$. Details of the construction procedure are provided in \cref{sec:append-mnist}.
\end{itemize}

\cref{fig:mnist-compare} shows the ground truth, posterior means, and maximum-a-posteriori (MAP) estimates for the MNIST inverse problem. MAP solutions are obtained by directly minimizing the negative logarithmic posterior defined in \cref{eq:mnist-post}, using the Adam optimizer with a sufficiently large number of iterations to ensure convergence. As illustrated in the figure, the resulting MAP estimates display substantial variability across runs, reflecting the nonconvexity of the posterior landscape. In contrast, the posterior means reported in Figure~\ref{fig:mnist-converge} achieve noticeably smaller reconstruction errors relative to the ground truth. 

We further report convergence results for \ApproxMH, \ProxMH, and NUTS in terms of the relative mean error and the relative componentwise second moment error; \ProxMH uses a single Gauss–Newton step. Compared to \Approx, \ProxMH exhibits lower error with faster convergence, reduced variability across trials, and robustness to different observation realizations.

\Cref{fig:mnist-beta} illustrates the effect of $\beta$ on \ProxMH. Consistent with bimodal experiments (\cref{fig:2well-beta}), higher acceptance rates are associated with smaller relative mean errors, but performance in this test is largely insensitive to $\beta$, with a mean error variance below $0.4\%$. We further examine the variability of the Jacobian determinant associated with a single Gauss--Newton correction in \cref{fig:logdet}.

\begin{figure}
    \centering
\begin{tikzpicture}
    \node[inner sep=0pt] (a) at (0,0)
  {\includegraphics[width=0.24\textwidth]{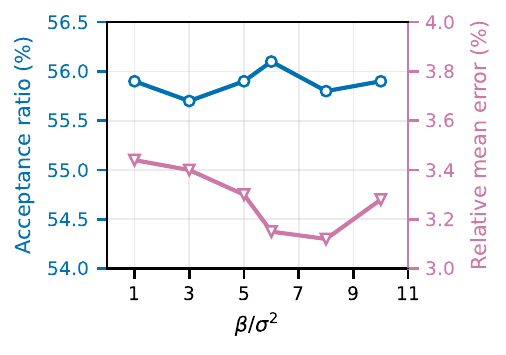}};
  \node[inner sep=0pt] (b)  at (0.24\textwidth,0)
  {\includegraphics[width=0.24\textwidth]{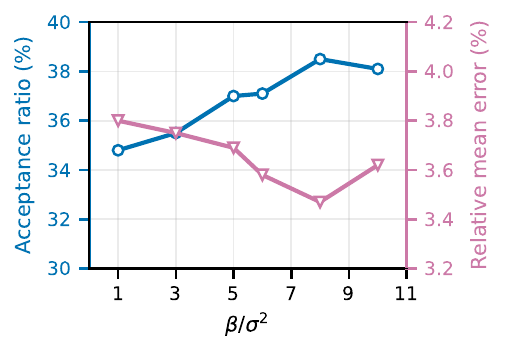}};
\node[anchor=north,yshift=-1pt] at (a.south) {\scriptsize \textbf{Test IV} (6.1\%)};
\node[anchor=north,yshift=-1pt] at (b.south) {\scriptsize \textbf{Test V} (6.1\%)};
\end{tikzpicture}
    \caption{Effect of the hyperparameter $\beta$ on \ProxMH in the inverse problem in MNIST. Values in parentheses represent $\|\A-\Aapprox\|/\|\A\|$. Blue curves show acceptance rates (left y-axis), and pink curves show relative mean errors after $10^5$ MH steps (right y-axis). The results are not sensitive to $\beta$.}
    \label{fig:mnist-beta}
\end{figure}

\subsection{Nonlinear Helmholtz test.}\label{sec:test-nonlinear}

We consider a nonlinear inverse problem arising from acoustic scattering governed by the Helmholtz equation. The unknown medium $x$ is represented on a $32\times 32$ grid, while the wavefield is discretized on a finer grid $128\times 128$. We consider $4$ independent source terms. The forward operator $\A_i(x)$ is defined implicitly by solving a Helmholtz equation with homogeneous Neumann boundary conditions (see \cref{sec:append-nonlinear} for details). The Helmholtz operator is nonlinear with respect to the scatterer properties. Assuming independent observations, the negative log-posterior $-\log \pi(x \mid \{y_i\})$ (up to an additive constant) is
\begin{align}
\frac{1}{2\sigma^2}\sum_{i=1}^{4}\|y_i-\A_i(x)\|^2 + TV_{\epsilon}(x),
\end{align}
where $TV_{\epsilon}(x)$ denotes a smoothed total-variation prior.


To construct an approximate forward model, we solve the Helmholtz equation on a coarser $64\times 64$ grid.
All linear systems arising in the forward and adjoint solves are computed using GMRES with a Neumann--DCT preconditioner; implementation details are provided in \cref{sec:append-nonlinear}.

\cref{fig:nonlinear} shows the ground truth, the MAP estimate, and the posterior mean. The MAP estimate is closer to the ground truth because the smoothed TV prior is weakly informative, allowing the likelihood to dominate the posterior; this observation is further corroborated by \cref{fig:u-plots} in the appendix, where the corresponding wavefield solutions $u_i$ for the ground truth, MAP, and posterior mean are nearly indistinguishable. From the convergence results in \cref{fig:nonlinear}, \ProxMH consistently outperforms \ApproxMH and achieves accuracy comparable to NUTS, while retaining the computational advantages of independent proposals.

\begin{figure}[t]
\centering
\scriptsize
\begin{tikzpicture}
\def\W{0.3\columnwidth}  
\def\xgap{0.1\columnwidth} 
\def\ygap{1mm}              
\node[inner sep=0pt] (a)
  {\includegraphics[width=\W]{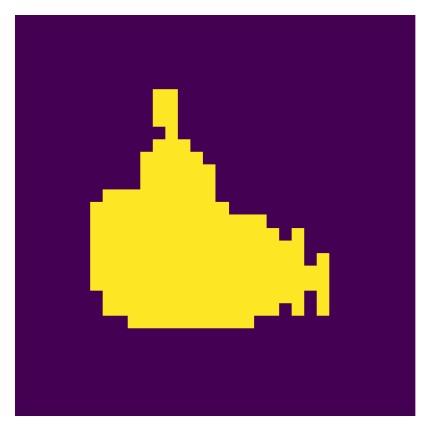}};
\node[inner sep=0pt, right=\xgap of a,xshift=2mm] (b)
  {\includegraphics[width=\W]{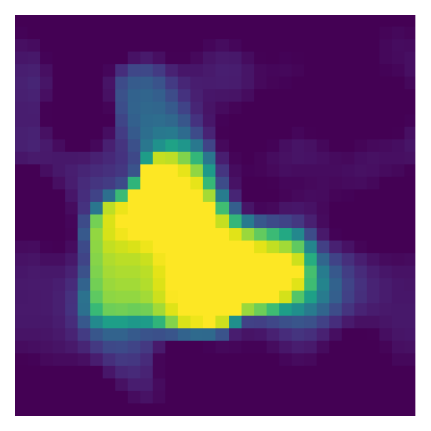}};

\node[
  rotate=90,
  anchor=center,
  font=\bfseries
] at ([xshift=-3mm]a.west) {Ground Truth};
\node[
  rotate=90,
  anchor=center,
  font=\bfseries
] at ([xshift=-3mm]b.west) {MAP};
\node[inner sep=0pt, below=\ygap of a] (c)
  {\includegraphics[width=\W]{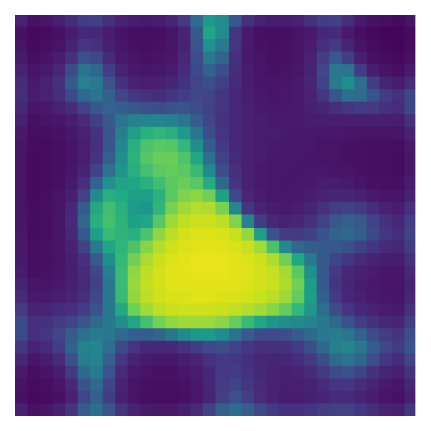}};
\node[inner sep=0pt, right=\xgap of c,xshift=-5mm, yshift=-2mm] (d)
  {\includegraphics[width=0.4\columnwidth]{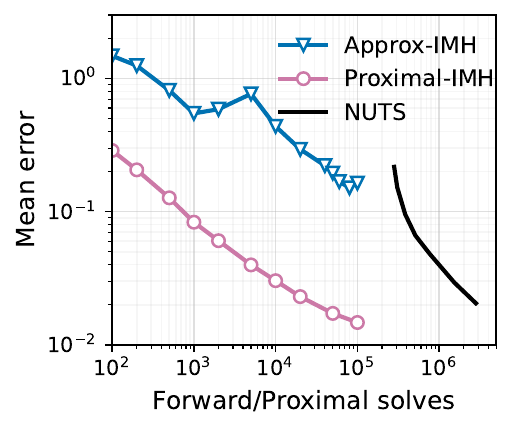}};

\node[
  rotate=90,
  anchor=center,
  font=\bfseries
] at ([xshift=-3mm]c.west) {Posterior Mean};

\end{tikzpicture}
\caption{Nonlinear Helmholtz inverse problem. Ground truth, MAP estimate, and posterior mean for the acoustic scattering test, together with convergence of the relative mean error.}
\label{fig:nonlinear}
\end{figure}
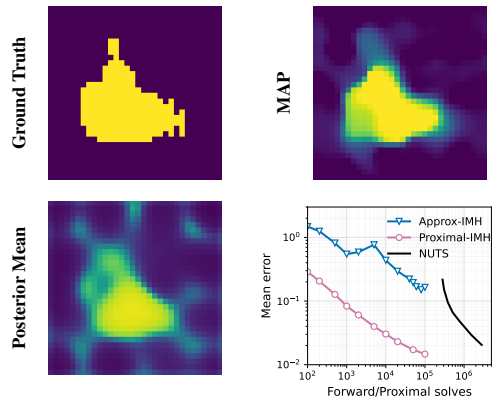

\section{Conclusion}\label{sec:conclusion}
We introduced \ProxMH, an independence Metropolis--Hastings sampler for inverse problems with approximate forward operators. \ProxMH generates global proposals by correcting samples from an approximate posterior through a stable optimization step. For linear forward models, we showed that \ProxMH achieves smaller expected KL divergence and faster mixing than existing IMH schemes under suitable regularity assumptions. Numerical experiments on multimodal, imaging, and nonlinear PDE-based inverse problems consistently confirm improved acceptance rates, convergence, and accuracy.
\paragraph{Limitations.}
Theoretical guarantees are derived under linear forward models and regularity assumptions on the prior, while the nonlinear extension relies on local Gauss--Newton corrections and neglects Jacobian terms in practice. Performance also depends on the quality of the approximate operator $\Fapprox$ and the choice of the regularization parameter $\beta$ in some cases.


\clearpage
\section*{Impact Statement}
This paper aims to advance the field of Machine Learning. While our work may have various societal impacts, none require specific mention here.

\section*{Acknowledgements}
This material is based upon work supported by NSF award OAC 22042261; and  Cooperative Agreement 2421782 and the Simons Foundation award MPS-AI-00010515 (NSF-Simons AI Institute for Cosmic Origins---CosmicAI, \url{https://www.cosmicai.org/}); by the U.S. Department of Energy, Office of Science, Office of Advanced Scientific Computing Research, Applied Mathematics program, Mathematical Multifaceted Integrated Capability Centers (MMICCS) program, under award number DE-SC0023171; by the U.S. Department of Energy, National Nuclear Security Administration Award Number DE-NA0003969; and by the U.S. National Institute on Aging under award number  R21AG074276-01. Any opinions, findings, and conclusions or recommendations expressed herein are those of the authors and do not necessarily reflect the views of the DOE, NIH, and NSF. Computing time on the Texas Advanced Computing Centers Stampede system was provided by an allocation from TACC and the NSF.

\bibliography{reference}

@article{chen-biros26,
    author = {Youguang Chen and George Biros},
    title = {{Latent-IMH}: Efficient {Bayesian} Inference for Inverse Problems with Approximate Operators},
    journal = {arXiv preprint},
    year = 2026
}

@article{dwivedi2019,
  title={Log-concave sampling: Metropolis-Hastings algorithms are fast},
  author={Dwivedi, Raaz and Chen, Yuansi and Wainwright, Martin J and Yu, Bin},
  journal={Journal of Machine Learning Research},
  volume={20},
  number={183},
  pages={1--42},
  year={2019}
}

@article{lecun1998gradient,
  title={Gradient-based learning applied to document recognition},
  author={LeCun, Yann and Bottou, L{\'e}on and Bengio, Yoshua and Haffner, Patrick},
  journal={Proceedings of the IEEE},
  volume={86},
  number={11},
  pages={2278--2324},
  year={1998},
  publisher={IEEE}
}

@article{kaipio-e00,
   Author = {Kaipio, J. P. and Kolehmainen, V. and Somersalo, E. and Vauhkonen, M.},
   Title = "{Statistical inversion and Monte Carlo sampling methods in electrical impedance tomography}",
   Journal = {Inverse Problems},
   Volume = {16},
   Number = {5},
   Pages = {1487-1522},
   Year = {2000} }

@BOOK{colton-kress98,
  title = {Inverse Acoustic and Electromagnetic Scattering Theory, 2nd Edition},
  publisher = {Springer},
  year = {1998},
  author = {David Colton and Rainer Kress},
  series = {Applied Mathematical Sciences}
}

@article{saratoon-arridge-e13,
  title={A gradient-based method for quantitative photoacoustic tomography using the radiative transfer equation},
  author={Saratoon, T and Tarvainen, T and Cox, BT and Arridge, SR},
  journal={Inverse Problems},
  volume={29},
  number={7},
  pages={075006},
  year={2013},
  publisher={IOP Publishing}
}

@article{arridge-schotland99,
  title={Optical tomography: forward and inverse problems},
  author={Arridge, Simon R and Schotland, John C},
  journal={Inverse Problems},
  volume={25},
  number={12},
  pages={123010},
  year={2009},
  publisher={IOP Publishing}
}

@book{tarantola-05,
  title = {Inverse problem theory and methods for model parameter estimation},
  publisher = {Society for Industrial and Applied Mathematics (SIAM)},
  year = {2005},
  author = {Tarantola, Albert},
  pages = {xii+342},
  address = {Philadelphia, PA},
  isbn = {0-89871-572-5}
}

@book{vogel02,
  title={Computational methods for inverse problems},
  author={Vogel, C.R.},
  volume={23},
  year={2002},
  publisher={SIAM}
}

@article{ghattas2021learning,
  title={Learning physics-based models from data: perspectives from inverse problems and model reduction},
  author={Ghattas, Omar and Willcox, Karen},
  journal={Acta Numerica},
  volume={30},
  pages={445--554},
  year={2021},
  publisher={Cambridge University Press}
}

@article{girolami2011riemann,
  title={Riemann manifold {L}angevin and {Hamiltonian Monte Carlo} methods},
  author={Girolami, Mark and Calderhead, Ben},
  journal={Journal of the Royal Statistical Society: Series B (Statistical Methodology)},
  volume={73},
  number={2},
  pages={123--214},
  year={2011},
  publisher={Wiley Online Library}
}

@article{hoffman2014no,
  title={The {No-U-Turn} sampler: adaptively setting path lengths in {Hamiltonian Monte Carlo}},
  author={Hoffman, Matthew D and Gelman, Andrew and others},
  journal={Journal of Machine  Learning Research},
  volume={15},
  number={1},
  pages={1593--1623},
  year={2014}
}

@article{inverse-song-ermon21,
  title={Solving inverse problems in medical imaging with score-based generative models},
  author={Song, Yang and Shen, Liyue and Xing, Lei and Ermon, Stefano},
  journal={arXiv preprint arXiv:2111.08005},
  year={2021}
}

@article{papamakarios2021normalizing,
  title={Normalizing flows for probabilistic modeling and inference},
  author={Papamakarios, George and Nalisnick, Eric and Rezende, Danilo Jimenez and Mohamed, Shakir and Lakshminarayanan, Balaji},
  journal={Journal of Machine Learning Research},
  volume={22},
  number={57},
  pages={1--64},
  year={2021}
}

@article{spantini2015optimal,
  title={Optimal low-rank approximations of Bayesian linear inverse problems},
  author={Spantini, Alessio and Solonen, Antti and Cui, Tiangang and Martin, James and Tenorio, Luis and Marzouk, Youssef},
  journal={SIAM Journal on Scientific Computing},
  volume={37},
  number={6},
  pages={A2451--A2487},
  year={2015},
  publisher={SIAM}
  }

@book{biegler2010large,
  title={Large-scale inverse problems and quantification of uncertainty},
  author={Biegler, Lorenz and Biros, George and Ghattas, Omar and Heinkenschloss, Matthias and Keyes, David and Mallick, Bani and Tenorio, Luis and van Bloemen Waanders, Bart and Willcox, Karen and Marzouk, Youssef},
  volume={712},
  year={2010},
  publisher={Wiley Online Library}
}

@book{book-kaipio2006,
  title={Statistical and computational inverse problems},
  author={Kaipio, Jari and Somersalo, Erkki},
  volume={160},
  year={2006},
  publisher={Springer Science \& Business Media}
}

@inproceedings{cohen2014solving,
  title={Solving SDD linear systems in nearly m log1/2 n time},
  author={Cohen, Michael B and Kyng, Rasmus and Miller, Gary L and Pachocki, Jakub W and Peng, Richard and Rao, Anup B and Xu, Shen Chen},
  booktitle={Proceedings of the forty-sixth annual ACM symposium on Theory of computing},
  pages={343--352},
  year={2014}
}

@book{van2003iterative,
  title={Iterative Krylov methods for large linear systems},
  author={Van der Vorst, Henk A},
  number={13},
  year={2003},
  publisher={Cambridge University Press}
}

@article{boomeramg00,
 author = {Van Emden Henson and Ulrike Meier Yang},
 title = {BoomerAMG: a parallel algebraic multigrid solver and preconditioner},
 journal = {Appl. Numer. Math.},
 volume = {41},
 number = {1},
 year = {2002},
 issn = {0168-9274},
 pages = {155--177},
 publisher = {Elsevier Science Publishers B. V.},
 address = {Amsterdam, The Netherlands, The Netherlands},
 }

@article{mala-roberts1996,
  title={Exponential convergence of {L}angevin distributions and their discrete approximations},
  author={Roberts, Gareth O. and Tweedie, Richard L.},
  year={1996}
}

@article{2stage,
author = {Peherstorfer, Benjamin and Willcox, Karen and Gunzburger, Max},
title = {Survey of Multifidelity Methods in Uncertainty Propagation, Inference, and Optimization},
journal = {SIAM Review},
volume = {60},
number = {3},
pages = {550-591},
year = {2018},
doi = {10.1137/16M1082469},
URL = {https://doi.org/10.1137/16M1082469},
}

@rticle{liang2020deep,
  title={Deep magnetic resonance image reconstruction: Inverse problems meet neural networks},
  author={Liang, Dong and Cheng, Jing and Ke, Ziwen and Ying, Leslie},
  journal={IEEE Signal Processing Magazine},
  volume={37},
  number={1},
  pages={141--151},
  year={2020},
  publisher={IEEE}
}

@book{epstein2007introduction,
  title={Introduction to the mathematics of medical imaging},
  author={Epstein, Charles L},
  year={2007},
  publisher={SIAM}
}

@article{herrmann2024neural,
  title={Neural and spectral operator surrogates: unified construction and expression rate bounds},
  author={Herrmann, Lukas and Schwab, Christoph and Zech, Jakob},
  journal={Advances in Computational Mathematics},
  volume={50},
  number={4},
  pages={72},
  year={2024},
  publisher={Springer}
}
\bibliographystyle{icml2026}

\newpage
\appendix
\onecolumn
\section{Proofs of Main Results}

\subsection{Proof of \cref{thm:kl-simple}}\label{sec:proof-kl}

Under the Gaussian setting, the expected KL divergence $\KLapprox$ admits the following closed-form expression:
{\setlength{\abovedisplayskip}{6pt}
 \setlength{\belowdisplayskip}{6pt}
 \setlength{\abovedisplayshortskip}{6pt}
 \setlength{\belowdisplayshortskip}{6pt}
\begin{align}
\KLapprox
&=
\begin{tikzpicture}[baseline=(X.base)]
\node (X) {$
\log\frac{|\bSigma|}{|\bSigma_a|}
+ \Tr(\bSigma^{-1} \bSigma_a) - d
$};
\draw[decorate,decoration={brace,mirror},
      color=tabblue, line width=0.8pt]
  ([xshift=0ex]X.south west) --
  ([xshift=0ex]X.south east)
  node[midway, yshift=-6pt, text=tabblue]
  {\scriptsize covariance mismatch};
\end{tikzpicture}
\label{eq:kl-approx}\\
&\quad +
\begin{tikzpicture}[baseline=(Y.base)]
\node (Y) {$
\frac{1}{\sigma^2}\|\A \daapprox \A \|_F^2
+ \sigma^2 \|\daapprox \|_F^2
+ 2 \| \daapprox \A\|_F^2
$};
\draw[decorate,decoration={brace,mirror},
      color=tabred, line width=0.8pt]
  ([xshift=0ex]Y.south west) --
  ([xshift=0ex]Y.south east)
  node[midway, yshift=-6pt, text=tabred]
  {\scriptsize mean mismatch};
\end{tikzpicture},\nonumber
\end{align}
}
where $\daapprox := \A_a^\dagger - \A^\dagger$. The explicit forms of $\bSigma$, $\bSigma_a$, $\A^\dagger$, and $\A_a^\dagger$ are given in \cref{tab:posterior-gaussian}.

The expressions for $\KLlatent$ and $\KLprox$ follow the same structure as \cref{eq:kl-approx}, with $(\A_a^\dagger,\bSigma_a)$ replaced by $(\A_l^\dagger,\bSigma_l)$ and $(\A_p^\dagger,\bSigma_p)$, respectively.

\begin{proof}[Proof of \cref{thm:kl-simple}]
The expressions for $\KLapprox$ and $\KLlatent$ have been proved in Proposition 3.3 in \cite{chen-biros26}. We only need to prove the expression for $\KLprox$ here. 
\begin{align}\label{eq:pf-kl-prox}
    \KLprox= \log\frac{|\bSigma|}{|\bSigma_p|} + \Tr(\bSigma^{-1} \bSigma_p) - d +\frac{1}{\sigma^2}\|\A \daprox \A \|_F^2+ \sigma^2 \|\daprox \|_F^2 + 2 \| \daprox \A\|_F^2,
\end{align}
where $\bSigma$ and $\bSigma_p$ are the covariances of \Exact and \Proximal, and $\daprox:=\A_p^\dagger - \A^\dagger$ is the pseudo-inverse difference.

By the diagonal assumptions of $\F$, $\Fapprox$ and $\b O$, we have $\A, \Aapprox\in\mathbb{R}^{d_y \times d}$ are matrices with zero elements except for:
\begin{align}
    \A_{ii} = s_i, \qquad \Aapprox_{ii} =\alpha_i s_i, \quad \forall i\in[d_y].
\end{align}
By the formulations listed  in \cref{tab:posterior-gaussian}, $\A^\dagger, \A_a^\dagger, \A_p^\dagger \in \mathbb{R}^{d\times d_y}$ have all zero entries except for:
\begin{align}
    \A^\dagger_{ii} = \frac{s_i}{s_i^2 +\sigma^2}, \qquad \left(\A_a^\dagger\right)_{ii} = \frac{\alpha_i s_i}{\alpha_i^2 s_i^2 + \sigma^2} ,\qquad 
    \left(\A_p^\dagger \right)_{ii} = \frac{\alpha_i^2 s_i^3 + \sigma^2}{(s_i^2+\sigma^2)(\alpha_i^2 s_i^2 + \sigma^2)}, \quad \forall i \in[d_y].
\end{align}
Then  $\daprox\in\mathbb{R}^{d\times d_y}$ is matrix with zero elements except for:
\begin{align}\label{eq:pf-kl-daprox}
    (\daprox)_{ii} = (\A_p^\dagger)_{ii} - \A^\dagger_{ii} = \frac{\sigma^2 s_i (\alpha_i - 1)}{(\alpha_i^2s_i^2 + \sigma^2)(s_i^2 + \sigma^2)}, \quad \forall i\in[d_y].
\end{align}
In addition, the covariance of the posterior $\bSigma, \bSigma_p \in\mathbb{R}^{d\times d}$ are diagonal matrices with diagonals as  
\begin{align}
   \begin{cases}\label{eq:kl-covariance}
      \bSigma_{ii}= \frac{\sigma^2}{s_i^2 +\sigma^2},\qquad (\bSigma_p)_{ii} = \frac{ \sigma^2 \rho_i}{ s_i^2 + \sigma^2}, \quad i\leq d_y,\\
        \bSigma_{ii}= 1,\qquad (\bSigma_p)_{ii} = 1,\qquad i >d_y.
    \end{cases}
\end{align}
Thus, for the covariance mismatch of \cref{eq:pf-kl-prox}, we have 
\begin{align}\label{eq:kl-prox-1}
    \log\frac{|\bSigma|}{|\bSigma_p|} = \sum_{i\in[d_y]} \log \frac{1}{\rho_i},\quad \Tr(\bSigma^{-1} \bSigma_p) = \sum_{i\in[d_y]} \rho_i + d- d_y.
\end{align}
Based on \cref{eq:pf-kl-daprox}, we have
\begin{align}
    \left(\A \daprox\right)_{ii} = \frac{\sigma^2 s_i^2(\alpha_i -1)}{(\alpha_i^2 s_i^2 +\sigma^2) (s_i^2 + \sigma^2)}
\end{align}
Define $\psi_i = (\alpha_i^2 s_i^2 +\sigma^2) (s_i^2 + \sigma^2)$, we have formulations for the  mean mismatch terms in \cref{eq:pf-kl-prox}:
\begin{align}
    \frac{1}{\sigma^2} \| \A \daprox \A\|_F^2 = \sum_{i\in[d_y]}\frac{\sigma^2 s_i^6 (\alpha_i-1)^2}{\psi_i^2}, \quad \sigma^2\|\daprox\|_F^2 = \sum_{i\in[d_y]}\frac{\sigma^6 s_i^2 (\alpha_i-1)^2}{\psi_i^2},\quad
    2\|\daprox \A\|_F^2 = \sum_{i\in[d_y]}\frac{2 \sigma^4 s_i^4 (\alpha_i -1)^2}{\psi_i^2}.
\end{align}
Thus, the mean mismatch term is 
\begin{align}\label{eq:kl-prox-2}
    \frac{1}{\sigma^2}\|\A \daprox \A \|_F^2+ \sigma^2 \|\daprox \|_F^2 + 2 \| \daprox \A\|_F^2 = \sum_{i\in[d_y]}\frac{\sigma^2 s_i^2 (\alpha_i-1)^2}{(\alpha_i^2 s_i^2 + \sigma^2)^2} =\sum_{i\in[d_y]} \zeta_i(\alpha_i-1)^2 s_i^2 \sigma ^2. 
\end{align}
Substituting \cref{eq:kl-prox-1,eq:kl-prox-2} into \cref{eq:pf-kl-prox}, we can get the expression for $\KLprox$ in \cref{thm:kl-simple}.
\end{proof}

\subsection{Mixing Time}\label{sec:proof-mix-time}

\begin{assumption}\label{assume:log-concave}
    The exact posterior $\pi(x|y)$ is strongly log-concave, i.e., $-\log \pi(x|y)$ is $m$-strongly convex on $\mathbb{R}^d$.
\end{assumption}

\begin{assumption}\label{assume:lipschitz}
  Log-weight functions are locally Lipschitz. For any $R>0$, there exist constants $C_a(R) \geq 0$ such that for all $(x,x') \in \mathsf{Ball}(x^*,R)$,  $\|\wapprox(x) - \wapprox(x')\|\leq C_a(R)\|x-x'\|$.

   For the functions $\wnew(x)$ and $\wprox(x)$, we assume that this condition holds with constants $C_l(R)$ and $C_p(R)$, respectively.

\end{assumption}

\begin{theorem}[Adapted from Theorem~4.3 in~\cite{chen-biros26}]\label{thm:mix-general}
 Let $\epsilon\in(0,1)$. Assume that $\pi_a(x \mid y)$, $\pi_l(x\mid y)$ and $\pi_p(x\mid y)$ are all $\gamma$-warm starts with respect to the exact posterior $\pi(x|y)$  (i.e., for any Borel set $\mathcal{E}$, $\mu_0(\mathcal{E}) \leq \gamma \pi(\mathcal{E})$).
 Assume that \cref{assume:log-concave} and \cref{assume:lipschitz} hold. For \cref{assume:lipschitz}, suppose that the Lipschitz conditions hold with $C_a({R_a})$, $C_l({R_l})$ and $C_p({R_p})$ $\leq\log 2 \sqrt{m}/32 $, where 
 \begin{align}
R_a = \max\left\{
\sqrt{\frac{d}{m}}r\!\left(\frac{\epsilon}{17\gamma}\right),
\sqrt{\frac{d}{m}}r\!\left(\frac{\epsilon}{272\gamma}\right)+\|x^*-x_a^*\|
\right\}.
 \end{align}
 \begin{align}
     R_l &= \max\left\{\sqrt{\frac{d}{m}}r\left(\frac{\epsilon}{17\gamma}\right), \quad \sqrt{\frac{d}{m}}r\left(\frac{\epsilon}{272\gamma}\right)\frac{1}{\sigma_{\min}(\Fapprox^{-1} \F)}+ \|x^* - \xnew^*\|\right\},
 \end{align}
  \begin{align}\label{eq:radius-proximal}
     R_p &= \max\left\{\sqrt{\frac{d}{m}}r\left(\frac{\epsilon}{17\gamma}\right), \quad \sqrt{\frac{d}{m}}r\left(\frac{\epsilon}{272\gamma}\right){\sigma_{\max}(\K)}+ \|x^* - x_p^*\|\right\},
 \end{align}
where  $\sigma_{\min}(\cdot)$ and $\sigma_{\max}(\cdot)$ are the smallest and largest singular values of the matrix,  $r(\cdot)$ is a constant  defined by
\begin{align}\label{eq:r-def}
    r(s)  = 2 + 2 \max\left\{\frac{-\log^{1/4}(s)}{d^{1/4}}, \frac{-\log^{1/2}(s)}{d^{1/2}} \right\}.
\end{align}
    Then the mixing time for \ApproxMH has
    \begin{align}\label{eq:mix-approx}
       \mixapprox(\epsilon) \leq 128 \log \left(\frac{2\gamma}{\epsilon}\right)\max\left(1,\frac{128^2 C_a^2(R_a)}{(\log 2)^2 m}\right).
    \end{align}
The mixing time for \LatentMH and \ProxMH, $\mixlatent( \epsilon)$ and $\mixprox(\epsilon)$  satisfy a similar bound with $C_a(R_a)$  replaced by $C_l(R_l)$ and $C_p(R_p)$ respectively.
\end{theorem}
\begin{proof}
The results for \ApproxMH and \LatentMH were established in Theorem~4.3 of \cite{chen-biros26}. 
We follow a similar argument to the proof of \ApproxMH to establish the result for \ProxMH.

Since
\[
\pi_p(x \mid y) \propto q\!\left(y - \Aapprox \K^{-1} x\right)\, p(\K^{-1} x),
\]
 Hence $\pi_p(x \mid y)$ is $m\,\sigma_{\min}^2(\K^{-1})$-strongly log-concave.
Consequently, there exists a constant $c_2>0$ such that
\[
\pi_p\!\left(\mathrm{Ball}(x_p^*, R_2)\right) \ge 1 - c_2 s,
\]
where
\begin{align}
R_2 = \sqrt{\frac{d}{m}}\, \sigma_{\max}(\K)\, r(c_2 s).
\end{align}

If $C_p(R_p) \le \frac{\log(2\sqrt{m})}{32}$, where $R_p$ satisfies \cref{eq:radius-proximal}, then the conditions of Equation~(96) in \cite{chen-biros26} are satisfied. 
The desired mixing-time bound for \ProxMH therefore follows.
\end{proof}

\begin{proof}[Proof of \cref{thm:mixtime}]
We prove the result by invoking \cref{thm:mix-general}. Our goal is to derive local Lipschitz constants
$C_a(R_a)$, $C_l(R_l)$, and $C_p(R_p)$ for log-weight functions $w_a(x)$, $w_l(x)$, and $w_p(x)$,
respectively. Without loss of generality, we assume that the posterior mode is
\[
x^\star := \arg\min_x \, \pi(x \mid y)
\]
is located at the origin, i.e., $x^\star = 0$. Throughout, we denote by $B(R)$ the Euclidean ball
centered at the origin with radius $R>0$.

\begin{enumerate}
\item \ApproxMH: The log-weight function is 
\begin{align}\label{eq:pf-mix-a-1}
    w_a(x) = \log \frac{\pi(x|y)}{\pi_a(x|y)} &=- \frac{1}{2\sigma^2}\left(\|y-\A x\|^2 - \|y - \Aapprox x\|^2\right)\nonumber\\
    & = \frac{1}{2\sigma^2}\left(x^\top \underbrace{ (\A + \Aapprox)^\top \Delta \A}_{=:\M_a} x + 2 y^\top \Delta \A x\right).
\end{align}
Define the symmetric version of $\M_a$ by $\M_{a,s} = \frac{1}{2}\left(\M_a + \M_a^\top\right)$, then $x^\top \M_a x = x^\top \M_{a,s} x$. Thus for any $x_1, x_2 \in B(R_a)$
\begin{align}\label{eq:pf-mix-a-2}
   \left |  x_1^\top \M_{a,s} x_1 - x_2^\top \M_{a,s} x_2\right | &= \left | (x_1+ x_2)^\top \M_{a,s} (x_1 - x_2) \right| \leq \| \M_{a,s}\| \|x_1+x_2\| \|x_1-x_2\|\nonumber\\
   &\leq 4 R_a \| \A \| \| \Delta \A\| \|x_1-x_2\|.
\end{align}
Besides, we have
\begin{align}\label{eq:pf-mix-a-3}
    \left | 2y^\top \Delta \A x_1 - 2y^\top \Delta \A x_2 \right| \leq 2 \| \Delta \A^\top y\| \|x_1 - x_2\| \leq 2 \|y\| \|\Delta \A\| \|x_1 - x_2\|.
\end{align}
Combining \cref{eq:pf-mix-a-1,eq:pf-mix-a-2,eq:pf-mix-a-3}, we have
\begin{align}
    \left | w_a(x_1)  - w_a(x_2)\right | \leq \underbrace{\frac{2R_a\|\A\| + \|y\|}{\sigma^2} \|\Delta \A\|}_{=: C_a(R_a)}\|x_1-x_2\|.
\end{align}
When $d$ is large, by \cref{thm:mix-general}, $R_a \sim \sqrt{\frac{d}{m}}$, thus 
\begin{align}
    C_a(R_a) &\sim \sqrt{\frac{d}{m}}\frac{\|\A\|}{\sigma^2} \| \Delta \A\|,\nonumber\\
    \tau_{mix}^a& \sim \frac{d}{m^2} \frac{\|\A\|^2}{\sigma^4}\|\Delta \A\|^2.
\end{align}
\item \LatentMH: Denote $f$ as the negative log prior, i.e. $f(x) = -\log p(x)$, then the log-weight function is
\begin{align}\label{eq:pf-mix-l-1}
    w_l(x) = \log \frac{\pi(x|y)}{\pi_l(x|y)} = -\left(f(x) - f(\M_l x)\right), \qquad \M_l:= \Fapprox^{-1} \F.
\end{align}

Thus we have gradient of the log-weight function
\begin{align}\label{eq:pf-mix-l-2}
    \nabla w_l(x) &= \nabla f(x) - \M_l^\top \nabla f(\M_l x)\nonumber\\
    &= \left(\b I - \M_l^\top \right)\nabla f(x) + \M_l^\top \left ( \nabla f(x) - \nabla f(\M_l x)\right)
\end{align} 
By assumption in \cref{thm:mixtime}, $f$ is $L-$smooth, thus
\begin{align}\label{eq:pf-mix-l-3}
    \| \nabla w_l(x)\| &\leq \|\b I - \M_l\| \|\nabla f(x)\| + \left\| \M_l\right\| \left \|\nabla f(x) - \nabla f(\M_l x)\right\|\nonumber\\
    & \leq \|\b I - \M_l\| \|\nabla f(x)\|  + \|\M_l\| L \|x - \M_l x\| \nonumber\\
    &\leq \|\b I - \M_l\|  \left ( \| \nabla f(x)\| + L\|\M_l\| \|x\|\right)
\end{align}
For any $x_1, x_2 \in B(R_l)$, we have
\begin{align}\label{eq:pf-mix-l-4}
    \left |w_l(x_1) - w_l(x_2) \right| \leq \sup_{x\in B(R_l)} \| \nabla w_l(x)\| \|x_1 - x_2\|.
\end{align}
By \cref{eq:pf-mix-l-3}, we have
\begin{align}
    \sup_{x\in B(R_l)} \|\nabla w_l(x)\| \leq \left \|\b I - \M_l \right\| \left(\sup_{x\in B(R_l)} \| \nabla  f (x)\| + LR_l\|\M_l\| \right),
\end{align}
where 
\begin{align}
   \sup_{x \in B(R_l)} \|\nabla f(x)\| \leq \|\nabla f(0) \| + L\|x\| \leq \|\nabla f(0)\| + L R_l.
\end{align}
Substitute into \cref{eq:pf-mix-l-4}:
\begin{align}\label{eq:pf-mix-l-5}
     \left |w_l(x_1) - w_l(x_2) \right| \leq \underbrace{\|\b I - \M_l \| \left[\left(1 +\|\M_l\|\right) LR_l + \|\nabla f(0)\|\right]}_{=: C_l(R_l)} \|x_1 - x_2\|.
\end{align}
When $d$ is large, by \cref{thm:mix-general}, $R_l \sim \sqrt{\frac{d}{m}}$, thus 
\begin{align}
    C_l(R_l) &\sim L\sqrt{\frac{d}{m}} \left(1+ \|\M_l\|\right) \| \b I- \M_l\|,\nonumber\\
    \tau_{mix}^l& \sim \frac{L^2d}{m^2} \left(1+ \|\M_l\|\right)^2 \| \b I- \M_l\|^2.
\end{align}
\item \ProxMH:  Denote $f(x) = -\log p(x)$, by \cref{eq:proximal-linear}, we have \Proximal:
\begin{align}
    \pi_p(x|y) \propto \exp\left( -\frac{1}{2\sigma^2} \|\Aapprox \K^{-1} x - y\|^2 - f(\K^{-1} x)\right)
\end{align}
the log-weight function is 
\begin{align}\label{eq:pf-mix-p-1}
    w_p(x) = \log\frac{\pi(x|y)}{\pi_p(x|y)} =- \frac{1}{2\sigma^2} \underbrace{\left(\|\A x- y\|^2 - \|\Aapprox \K^{-1} x - y\|^2\right)}_{=:h(x)}  - \left( f(x) - f(\K^{-1} x)\right).
\end{align}
According to \cref{sec:linear}, given $\xapprox$, define $x = \K \xapprox$. Then $x$ is the solution to the following optimization:
\begin{align}
    \min_x \|\A x - \Aapprox \xapprox\|^2 + \beta \|x - \xapprox\|^2.
\end{align}
Taking gradient of the objective, we can get:
\begin{align}\label{eq:pf-mix-p-2}
    \A^\top r = -\beta (x - \xapprox), \qquad r = \A x - \Aapprox \xapprox.
\end{align}
Thus 
\begin{align}\label{eq:pf-mix-p-3}
    h(x) &= h (\K \xapprox) = \left \|\A \K \xapprox - y \right\|^2 - \left \| \Aapprox \xapprox - y\right\|^2\nonumber\\
    & = \left(\A \K \xapprox + \Aapprox \xapprox - 2 y\right)^\top \left(\A \K \xapprox - \Aapprox \xapprox\right) = \left(\A \K \xapprox + \A \xapprox + \Delta \A \xapprox- 2y \right)^\top (\A x - \Aapprox \xapprox)\nonumber\\
    &\overset{\text{\cref{eq:pf-mix-p-2}}}{=} -\beta( x +\xapprox)^\top (x - \xapprox) - 2 y^\top r + \xapprox^\top (\Delta \A)^\top r.
\end{align}
For $y$, we have $y$ generated by some $x_0$ by $y = \A x_0 + e$, where $e$ is the noise, we  neglect the noise term because we discuss the scaling of the mixing time. Thus by \cref{eq:pf-mix-p-2}
\begin{align}\label{eq:pf-mix-p-4}
    y^\top r = x_0^\top \A^\top r = -\beta x_0^\top (x-\xapprox).
\end{align}
Let $\A^\dagger = (\A^\top \A + \beta \b I)^{-1}\A^\top$, we have $\K = \b I + \A^\dagger \Delta \A$. For the last term in \cref{eq:pf-mix-p-3}, we have
\begin{align}\label{eq:pf-mix-p-5}
    \xapprox^\top (\Delta \A)^\top r &= \xapprox^\top  (\Delta \A)^\top\left( \A \K \xapprox - \Aapprox \xapprox \right) = \xapprox^\top  (\Delta \A)^\top\left( \A \left(\b I + \A^\dagger \Delta\A\right) - \Aapprox \right)\xapprox \nonumber\\
    &= \xapprox^\top (\Delta \A)^\top  \left(\A \A^\dagger - \b I\right) (\Delta \A) \xapprox .
\end{align}
Substituting \cref{eq:pf-mix-p-4,eq:pf-mix-p-5} into \cref{eq:pf-mix-p-3}, we can get
\begin{align}
    h(x) &= -\beta (x+ \K^{-1} x)^\top (x - \K^{-1} x) + 2 \beta x_0^\top (x- \K^{-1} x)+ (\K^{-1} x)^\top(\Delta \A)^\top  \left(\A \A^\dagger - \b I\right) (\Delta \A) (\K^{-1} x)\nonumber\\
    & = -\beta x^\top \left(\b I + \K^{-1}\right)^\top\left(\b I-\K^{-1}\right) x - 2 \beta x_0^\top (\b I - \K^{-1}) x + x^\top \left(\K^{-T} (\Delta \A)^\top  \left(\A \A^\dagger - \b I\right) (\Delta \A)\K^{-1}\right) x.
\end{align}
Thus, for any $x_1,x_2\in B(R_p)$, we have
\begin{align}\label{eq:pf-mix-p-6}
  \frac{1}{2\sigma^2}  \left |h(x_1) - h(x_2) \right| &\leq \frac{\beta}{2\sigma^2}\left \| \b I+\K^{-1}\right\| \left \| \b I-\K^{-1}\right\| \|x_1+x_2\| \|x_1-x_2\| \nonumber\\
  &+ \frac{\beta}{\sigma^2}\left \| \b I-\K^{-1}\right\| \|x_0\|\|x_1-x_2\| + \frac{\|\A \A^\dagger - \b I\|}{\sigma^2} \|\K^{-1}\|^2 \|\Delta \A\|^2 \|x_1+x_2\| \|x_1-x_2\|\nonumber\\
  & \lesssim R_p \left \| \b I+\K^{-1}\right\| \left \| \b I-\K^{-1}\right\| \|x_1-x_2\|,
\end{align}
where we derived the last equation by using $\beta= \Theta(\sigma^2)$.

As for the local Lipschitzness of another term in \cref{eq:pf-mix-p-1}, we adopt what we derived for the \LatentMH, i.e. by \eqref{eq:pf-mix-l-5}, we have for any $x_1, x_2 \in B(R_p)$,
\begin{align}\label{eq:pf-mix-p-7}
    \left |\left(f(x_1) -f (\K^{-1} x_1\right)  - \left(f(x_2) - f(\K^{-1} x_2)\right)\right| \leq \left \| \b I - \K^{-1} \right\|\left [ (1+\|\K^{-1}\|) LR_p + \|\nabla f(0)\| \right] \|x_1 - x_2\|.
\end{align}
Combining \cref{eq:pf-mix-p-1,eq:pf-mix-p-6,eq:pf-mix-p-7}, we have for any $x_1, x_2\in B(R_p)$
\begin{align}
    \left|w_p(x_1) - w_p(x_2) \right|\lesssim C_p(R_p)\|x_1-x_2\|,
\end{align}
where 
\begin{align}
    C_p(R_p) = (L+1) R_p (1+ \|\K^{-1}\|)\left \| \b I-\K^{-1}\right\|.
\end{align}
When $d$ is large, by \cref{thm:mix-general}, $R_p \sim \sqrt{\frac{d}{m}}$, and
\begin{align}
    \tau_{mix}^p \sim \frac{d(L+1)^2}{m^2} (1+ \|\K^{-1}\|)^2\left \| \b I-\K^{-1}\right\|^2.
\end{align}

\end{enumerate}

\end{proof}

\subsection{Derivation of \cref{eq:log-det-ratio}}

Let $\Aapprox G(x)$ denote an approximate forward model and assume that $\Delta \A = \delta \A$ for some $\delta \in (0,1)$. The residual is $r(\xapprox)=\A G(\xapprox)-\Aapprox G(\xapprox)=\delta \A\,G(\xapprox)$.
Then the Gauss--Newton update from \cref{eq:gn} becomes
\begin{align}\label{eq:pf-gn-1}
    GN(\xapprox):= \xapprox - \delta\left(\J(\xapprox)^\top \J(\xapprox) + \beta \mathbf{I}\right)^{-1}
\J(\xapprox)^\top \A G(\xapprox),
\end{align}

For small residuals $r$, we neglect second-order terms in the
linearization of the Gauss--Newton update \cref{eq:gn}.
Under this approximation, the Jacobian of a single Gauss--Newton step \cref{eq:pf-gn-1} takes the form 
\begin{align}
\J_{GN}(\xapprox)&= \b I - \delta \left(\J(\xapprox)^\top \J(\xapprox) + \beta \mathbf{I}\right)^{-1} \J(\xapprox)^\top  (\A \J _{G}(\xapprox))\nonumber\\
&=\b I - \delta \left(\J(\xapprox)^\top \J(\xapprox) + \beta \mathbf{I}\right)^{-1} \J(\xapprox)^\top \J(\xapprox) \nonumber\\
&=(1-\delta)\mathbf{I}+\delta \beta\left(\J^\top(\xapprox)\J(\xapprox) + \beta\mathbf{I}\right)^{-1},
\end{align}
where the second equality is derived by the fact that $\J(\xapprox) = \A \J_G(\xapprox)$.

For two arbitrary points $\xapprox_1$ and $\xapprox_2$, let
$\{\lambda_{1,i}\}_{i\in[d_x]}$ and $\{\lambda_{2,i}\}_{i\in[d_x]}$
denote the eigenvalues (in descending order) of
$\J^\top(\xapprox_1)\J(\xapprox_1)$ and
$\J^\top(\xapprox_2)\J(\xapprox_2)$, respectively.
The logarithm of the Jacobian determinant ratio can then be expressed as
\begin{align}
\log
\frac{|\det \J_{GN}(\xapprox_1)|}
     {|\det \J_{GN}(\xapprox_2)|}
=
\sum_{i\in[d_x]}
\log
\frac{1 - \delta \lambda_{1,i}/(\lambda_{1,i} + \beta)}
     {1 - \delta \lambda_{2,i}/(\lambda_{2,i} + \beta)}.
\end{align}

\section{Helmholtz Forward Model}\label{sec:helmholtz}
\subsection{Linearized Helmholtz (Born Approximation) Model}\label{sec:born}

Consider the Helmholtz equation
\begin{equation}
\mathcal{L}(x) u := \left(-\Delta - k^2 x\right) u = f,
\end{equation}
where $x$ denotes the spatially varying contrast parameter and $k>0$ is the wavenumber. Let $x_0$ be a known background medium and write
\[
x = x_0 + \delta x,
\]
where $\delta x$ is a small perturbation.

Let $u_0$ denote the background wavefield solving
\begin{align}
\mathcal{L}(x_0) u_0 = f.
\end{align}

Substituting this decomposition into the Helmholtz equation and retaining only first-order terms in $\delta x$ yields the linearized equation
\begin{align}
\mathcal{L}(x_0)\, \delta u
= k^2\, \delta x \, u_0,
\end{align}
where $\delta u := u - u_0$ denotes the scattered field.

The scattered field can therefore be expressed as
\begin{align}
\delta u
= \mathcal{L}(x_0)^{-1}
\!\left(k^2\, \delta x \, u_0\right).
\end{align}
This relation defines a linear map from the contrast perturbation $\delta x$ to the scattered field $\delta u$.
For convenience, we introduce the operator
\begin{align}\label{eq:linear-helmholtz-F}
F := \mathcal{L}(x_0)^{-1}
\!\left(k^2\, \mathrm{diag}(u_0)\right),
\end{align}
so that $\delta u = F\, \delta x$.

\subsection{Nonlinear Forward Model and Adjoint}
\label{sec:nonlinear-helmholtz}

For the full nonlinear problem, the forward model is given by
\begin{align}
\begin{cases}
    &\mathcal{L}(x) u = f, \quad \text{where } \mathcal{L}(x):=-\Delta - k^2 x\\
    &y = \b O u,
\end{cases}
\end{align}
where the wavefield $u = u(x)$ depends nonlinearly on the contrast parameter $x$.

Given a current parameter $x$, the forward wavefield $u$ is computed by solving the Helmholtz equation using GMRES in a matrix-free fashion, without explicitly forming the system matrix.

\paragraph{Adjoint formulation.}
Let $r := \b O u - y$ denote the data residual.
The adjoint field $v$ is defined as the solution of the adjoint Helmholtz equation
\begin{align}
\mathcal{L}(x)^\top v = \b O^\top r,
\end{align}
with adjoint boundary conditions consistent with the forward problem.

The gradient of the data misfit functional
\begin{align}
\Phi(x) := \tfrac12 \|\b O u - y\|^2
\end{align}
with respect to $x$ is given by
\begin{align}
\nabla_x \Phi(x)
= -k^2\, u \odot v,
\end{align}
where $\odot$ is the pointwise (Hadamard) product.

\section{Additional Experimental Details}

\begin{table*}[t]
\centering
    \caption{Distributions of \texttt{Exact}, \texttt{Approx}, \texttt{Latent} and \texttt{Proximal} posteriors for Gaussian prior $p(x) = \mathcal{N}(\b 0, \b I)$ and noise $q(e) = \mathcal{N}(\b 0, \sigma^2 \b I)$. For \Proximal, $\K = (\A^\top \A + \beta \b I)^{-1} (\A^\top \Aapprox+ \beta \b I)$.}
    \label{tab:posterior-gaussian}
\footnotesize
    \begin{tabular}{ccccc}
\toprule
Posterior & Distribution & Mean& Covariance & Pseudo-inverse\\\midrule
\texttt{Exact} &$ \pi(x \mid y)=\mathcal{N}(\mu, \bSigma)$  &$\mu =\A^\dagger y $ &$\bSigma=\b I - \A^\dagger \A$ & $\A^\dagger=\A^\top \left( \A\A^\top + \sigma^2 \b I\right)^{-1}$ \\
\texttt{Approx} &$ \piapprox(x \mid y)=\mathcal{N}(\mu_a, \bSigma_a)$  &$\mu_a =\A_a^\dagger y $ &$\bSigma_a=\b I - \A_a^\dagger \Aapprox$ & $\A_a^\dagger=\Aapprox^\top \left( \Aapprox\Aapprox^\top + \sigma^2 \b I\right)^{-1}$ \\
\texttt{Latent} &$ \pinew(x \mid y)=\mathcal{N}(\mu_l, \bSigma_l)$  &$\mu_l =\A_l^\dagger y $ &$\bSigma_l=\F^{-1} \Fapprox\bSigma_a  \Fapprox^\top \F^{-\top} $ & $\Anew^\dagger= \F^{-1} \Fapprox\A_a^\dagger$ \\
\texttt{Proximal} &$ \pi_p(x \mid y)=\mathcal{N}(\mu_p, \bSigma_p)$ & $\mu_p = \A_p^\dagger y$ & $\bSigma_p = \K \bSigma_a\K^\top$ & $\A_p^\dagger = \K \A_a^\dagger$\\
\bottomrule
    \end{tabular}
\end{table*}

\begin{table*}[t]
\small
    \centering
    \begin{tabular}{ccccc}
      \toprule  
  Effects  & $\log_{10}\mathrm{SNR}$ & $\frac{\| \F - \Fapprox\|}{\| \F\|}$  & $d_y/d$ & $d$\\
      \midrule
  Noise level     & 0.5 \textemdash 4.0 & 6\% & 0.2  & 500\\
 Operator error & 2.5& 2\% \textemdash 21\%& 0.2&500 \\
  Observation ratio &2.5  &6\% &0.05 \textemdash 0.5 & 500\\
  Dimension &2.5 &6\% &0.2 & 100 \textemdash 2000\\
         \bottomrule
    \end{tabular}
    \caption{Experimental setup for comparing $\KLapprox$, $\KLlatent$ and $\KLprox$.}
    \label{tab:kl-experiments}
\end{table*}

\subsection{Proposal comparisons in \cref{sec:kl}}\label{sec:append-kl}
In \cref{tab:posterior-gaussian}, we provide the Gaussian formulation of different posteriors when prior is normal and noise is Gaussian distribution with variance $\sigma^2$. \cref{tab:kl-experiments} presents the experimental setup for the sensitivity test on KL divergence in \cref{fig:kl}.

\subsection{Bimodal test in \cref{sec:2well}}
The histogram of the projection of prior samples in the bimodal test (see \cref{eq:prior-bimodal}) is plotted in \cref{fig:2well-prior}.

\begin{figure}[t]
    \centering
    \includegraphics[width=0.3\linewidth]{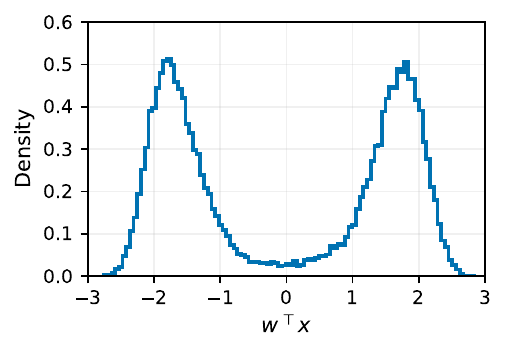}
    \caption{Histogram of the projection of prior samples in the bimodal test (see \cref{eq:prior-bimodal}).}
    \label{fig:2well-prior}
\end{figure}


\subsection{MNIST inverse problem in~\cref{sec:mnist}}\label{sec:append-mnist}

\paragraph{$\A$ and $\Aapprox$ construction in Tests~VI and~VII (\cref{sec:mnist}).}
We consider the linearized Helmholtz forward model introduced in \cref{sec:born}.
In our experiments, the parameter $x$ is represented on a $28\times 28$ grid corresponding to MNIST images, while the wavefield $u$ is discretized on a finer $100\times 100$ spatial grid. Consequently, the dimensions of the parameter and wavefield are $d_x = 784$ and $d_u = 10^4$, respectively. The wave number is set to $k=2.4 \pi$ in our experiments.

The linearized operator $F$ defined in \cref{eq:linear-helmholtz-F} is discretized on a two-dimensional uniform grid using a second-order finite-difference scheme with homogeneous Neumann boundary conditions; we denote the fine-grid inverse Helmholtz operator by $\F \in \mathbb{R}^{d_u \times d_u}$.
A sparse prolongation operator $\b P \in \mathbb{R}^{d_u \times d_x}$ maps the parameter $x$ from the coarse $28\times 28$ grid to the fine $100\times 100$ grid on which the Helmholtz equation is solved.
With an observation operator $\b O \in \mathbb{R}^{d_y \times d_u}$, the resulting fine-scale forward operator $\A \in \mathbb{R}^{d_y \times d_x}$ is defined as
\begin{align}
\A := \b O\, \F\, \b P.
\end{align}
The corresponding forward model from the latent variable $z$ to the observation $y$ is then given by
\begin{align}
y = \A\, G(z) + e,
\end{align}
where $G : \mathbb{R}^{d_z} \rightarrow \mathbb{R}^{d_x}$ denotes the decoder of the variational autoencoder (VAE).

To construct the approximate operator $\Aapprox$, we introduce coarser spatial grids of size $\widetilde{n}_u \times \widetilde{n}_u$ for the wavefield $u$, where $\widetilde{n}_u \in \{20, 35\}$ in our experiments.
We denote the corresponding discretized Helmholtz operator on the coarse grid by $\Fapprox \in \mathbb{R}^{\widetilde{d}_u \times \widetilde{d}_u}$, with $\widetilde{d}_u = \widetilde{n}_u^2$.
Sparse prolongation operators $\b P_u \in \mathbb{R}^{d_u \times \widetilde{d}_u}$ and $\b P_x \in \mathbb{R}^{\widetilde{d}_u \times d_x}$ are used to map vectors between the coarse and fine grids.
The resulting approximate forward operator is then defined as
\begin{align}\label{eq:linear-Aapprox-helm}
\Aapprox := \b O\, \b P_u\, \Fapprox\, \b P_x.
\end{align}

The dimensions involved in the construction of $\A$ and $\Aapprox$ are summarized in \cref{tab:dimension-mnist}, and schematic diagrams of the corresponding matrix factorizations are shown in \cref{fig:matrix-construction-mnist}.

\paragraph{Jacobian Determinant Diagnostics of the Gauss–Newton Correction.} Figure~\ref{fig:logdet} examines the variability of the Jacobian determinant associated with a single Gauss--Newton correction across posterior samples and operator approximation levels. In the left panels, we plot the sorted log determinants of the Gauss--Newton Jacobian evaluated at 2000 randomly selected posterior samples.
As the relative operator approximation error $\|\A-\Aapprox\|/\|\A\|$ increases, the log-determinant values exhibit a systematic shift and increased spread, indicating stronger local volume distortion induced by the correction map. The right plots assess relative variability by considering ratios of Jacobian determinants between 2000 randomly selected pairs of approximate posterior samples from $\pi_a(x \mid y)$.
Across all tests, these ratios remain tightly concentrated around one, with the 5\%--95\% quantile range remaining narrow even at larger approximation errors.
This suggests that, despite changes in absolute scale, the Gauss--Newton transformation induces relatively uniform local volume changes across the posterior in the test.

\begin{table}[t]
\caption{Dimensions of parameters in Tests VI and VII in \cref{sec:mnist}.}
\centering
\begin{tabular}{ccccc}
\toprule
$d_z$ & $d_x$ &$d_u$ & $\widetilde{d}_u$ &$d_y$\\
\midrule
128  & 784 {\color{gray}($=28^2$)} & $10^4$ {\color{gray}($=100^2$)} &$\{35^2,\, 20^2\}$ & 100\\
\bottomrule
\end{tabular}
\label{tab:dimension-mnist}
\end{table}

\begin{figure}[t]
\centering
\begin{tikzpicture}[
  scale=0.85,
  every node/.style={transform shape},
  mat/.style={draw, rounded corners=2pt, fill=black!3, inner sep=0pt},
  lab/.style={font=\small},
  eq/.style={font=\large}
]

\def\Hdy{0.7cm}    
\def\Hdx{1.3cm}    
\def\Hdut{1.4cm}   
\def\Hdu{2.8cm}    

\def\Wdx{1.3cm}    
\def\Wdut{1.4cm}   
\def\Wdu{2.8cm}    

\node[mat, minimum height=\Hdy, minimum width=\Wdx] (A)
  {\large$\mathbf A$};
\node[lab, below=2pt of A] {$d_y \times d_x$};

\node[eq, right=1mm of A] (EqL) {$=$};

\node[mat, minimum height=\Hdy, minimum width=\Wdu, right=1mm of EqL] (O)
  {\large$\mathbf O$};
\node[lab, below=2pt of O] {$d_y \times d_u$};

\node[mat, minimum height=\Hdu, minimum width=\Wdu, right=2mm of O] (F)
  {\large$\mathbf F$};
\node[lab, below=2pt of F] {$d_u \times d_u$};

\node[mat, minimum height=\Hdu, minimum width=\Wdx, right=2mm of F] (P)
  {\large$\mathbf P$};
\node[lab, below=2pt of P] {$d_u \times d_x$};

\begin{scope}[xshift=10cm]  

\node[mat, minimum height=\Hdy, minimum width=\Wdx] (Ahat)
  {\large$\Aapprox$};
\node[lab, below=2pt of Ahat] {$d_y \times d_x$};

\node[eq, right=1mm of Ahat] (EqR) {$=$};

\node[mat, minimum height=\Hdy, minimum width=\Wdu, right=1mm of EqR] (Oh)
  {\large$\mathbf O$};
\node[lab, below=2pt of Oh] {$d_y \times d_u$};

\node[mat, minimum height=\Hdu, minimum width=\Wdut, right=2mm of Oh] (Pu)
  {\large$\mathbf P_u$};
\node[lab, below=2pt of Pu] {$d_u \times \widetilde{d}_u$};

\node[mat, minimum height=\Hdut, minimum width=\Wdut, right=2mm of Pu] (Fhat)
  {\large$\Fapprox$};
\node[lab, below=2pt of Fhat] {$\widetilde{d}_u \times \widetilde{d}_u $};

\node[mat, minimum height=\Hdut, minimum width=\Wdx, right=2mm of Fhat] (Px)
  {\large$\mathbf P_x$};
\node[lab, below=2pt of Px] {$\widetilde{d}_u \times d_x$};

\end{scope}

\end{tikzpicture}
\vspace{-1mm}
\caption{
Matrix constructions of $\A$ and $\Aapprox$ in Tests VI and VII in \cref{sec:mnist}.
Left: $\A = \b O\,\F\,\b P$.
Right: $\Aapprox = \b O\,\b P_u\,\Fapprox\,\b P_x$. In our setting $d_y < d_x <d_u$, and $\widetilde{d}_u <d_u$. 
}
\label{fig:matrix-construction-mnist}
\end{figure}
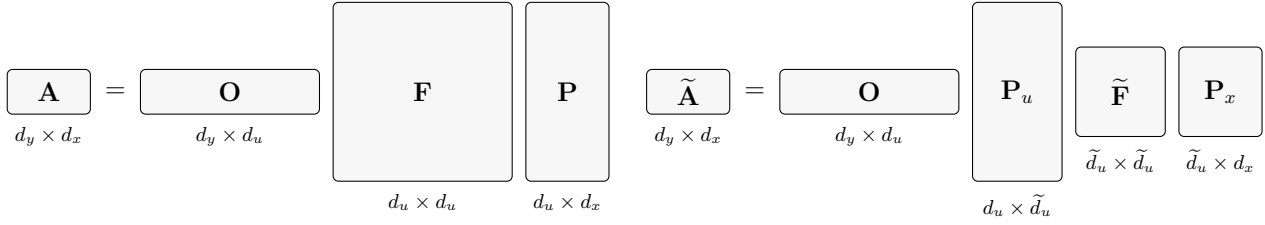

\begin{figure}[t]
\centering
\scriptsize
\begin{tikzpicture}
\node[inner sep=0pt] (a)
  {\includegraphics[width=0.24\textwidth]{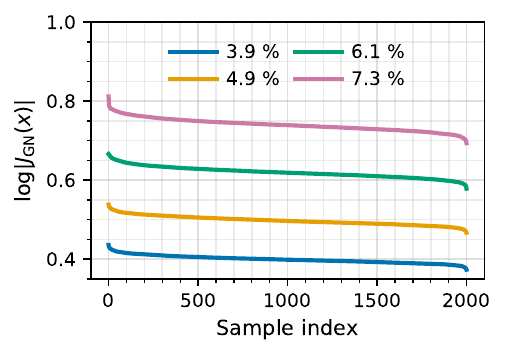}};
\node[inner sep=0pt, right=0.01\textwidth of a] (b)
  {\includegraphics[width=0.24\textwidth]{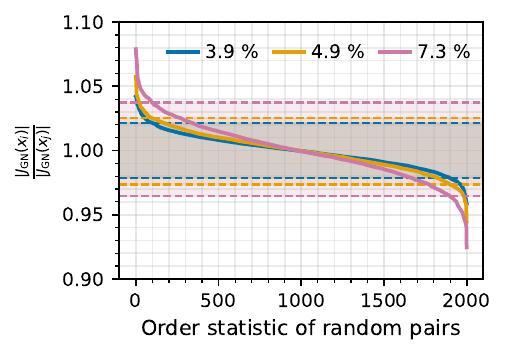}};

\node[inner sep=0pt, below=2mm of a] (c)
  {\includegraphics[width=0.24\textwidth]{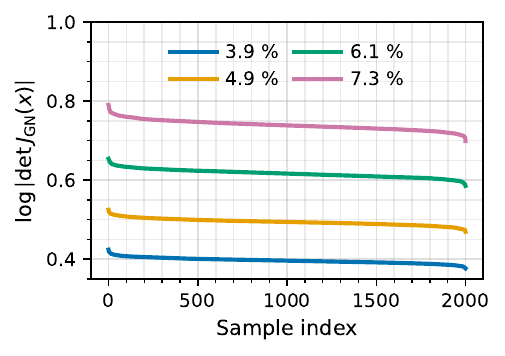}};
\node[inner sep=0pt, below=2mm of b] (d)
  {\includegraphics[width=0.24\textwidth]{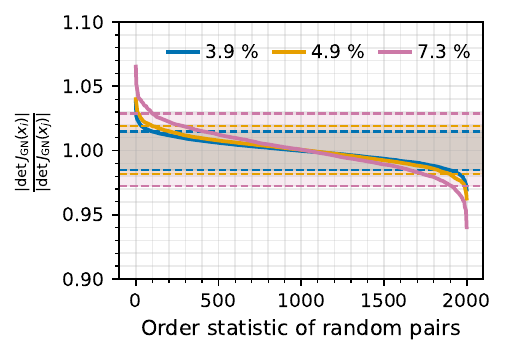}};

\node[rotate=90, anchor=south, font=\small]
  at ($(a.west)+(-2mm,0)$) {Test IV};

\node[rotate=90, anchor=south, font=\small]
  at ($(c.west)+(-2mm,0)$) {Test V};

\end{tikzpicture}
\vspace{-1mm}
\caption{Jacobian log-determinant diagnostics for the Gauss--Newton correction. Each row corresponds to one test case.
\textbf{Left:} Sorted values of the log determinant of the Jacobian of a single Gauss--Newton transformation, evaluated at 2000 randomly selected posterior samples from $\pi_a(x\mid y)$. \textbf{Right:} Order statistics of ratios of Jacobian determinants between 2000 randomly selected pairs of posterior samples from $\pi_a(x\mid y)$, i.e.,
$\det J_{\mathrm{GN}}(x_i) / \det J_{\mathrm{GN}}(x_j)$.
The numbers shown in the legends indicate the relative operator approximation error $\|\A - \Aapprox\| / \|\A\|$.
In the right panels, dashed lines denote the 5\%--95\% quantile range.
}
\label{fig:logdet}
\end{figure}

\subsection{Nonlinear Helmholtz test in \cref{sec:test-nonlinear}}\label{sec:append-nonlinear}
\paragraph{Forward operator $\A_i(x)$ and its approximation $\Aapprox_i(x)$.}
The nonlinear forward model and its adjoint are introduced in \cref{sec:nonlinear-helmholtz}.
The nonlinear Helmholtz operator is defined as
\[
\mathcal{L}(x) := -\Delta - k^2 x .
\]
The medium parameter $x$ is represented on a $32\times 32$ uniform grid over the domain $[0,1]\times[0,1]$.
The Helmholtz operator $\mathcal{L}(\cdot)$ is discretized on a finer $128\times 128$ uniform grid using a second-order finite-difference scheme with homogeneous Neumann boundary conditions; we denote the resulting nonlinear discretized operator by $\b L(\cdot)$. The wavenumber is set to $k=3.4\pi$ in this test.

We consider $n_s=4$ distinct source terms $\{f_i\}_{i\in[n_s]}$.
For each source, the corresponding observation is generated by
\begin{align}
    y_i = \b O\, \b L^{-1}(\b P x)\, f_i + e_i, 
    \qquad i\in[n_s],
\end{align}
where $\b P$ denotes a prolongation operator mapping the medium parameter from the coarse grid to the fine grid, and $\b O$ is the observation operator.
Accordingly, the nonlinear forward operator is defined as
\begin{align}\label{eq:nonlinear-A}
    \A_i(x) := \b O\, \b L^{-1}(\b P x)\, f_i,
    \qquad i\in[n_s].
\end{align}

To construct an approximate forward model, we introduce a coarser $64\times 64$ grid for the wavefield discretization.
Let $\widetilde{\b L}(\cdot)$ denote the discretized nonlinear Helmholtz operator on this coarse grid.
Analogous to the linear construction in \cref{eq:linear-Aapprox-helm}, we define the approximate nonlinear forward operator as
\begin{align}\label{eq:nonlinear-Aapprox}
    \Aapprox_i(x) := \b O\, \b P_u\, \widetilde{\b L}^{-1}(\b P_x x)\, f_i,
    \qquad i\in[n_s],
\end{align}
where $\b P_x$ and $\b P_u$ are sparse prolongation operators used to map the medium parameter and wavefield, respectively, from coarse grids to the fine grid. \cref{tab:dimension-nonlinear} summarizes the dimensions of parameters involved in this test.

\paragraph{GMRES solver with Neumann--DCT preconditioning.}
We solve the (nonlinear) Helmholtz linear systems arising in \cref{eq:nonlinear-A} and \cref{eq:nonlinear-Aapprox} and related adjoint computations using GMRES with a matrix-free spectral preconditioner. The preconditioner is based on a constant-coefficient approximation of the Helmholtz operator,
\[
\mathcal{M} \;:=\; -\Delta + k^2\bigl(1+\bar x\bigr),
\]
where $\bar x$ denotes the spatial average of the medium parameter $x$ and homogeneous Neumann boundary conditions are imposed.
On a uniform $n\times n$ grid, the Neumann Laplacian is diagonalized by the two-dimensional discrete cosine transform (DCT).
Let $\Lambda$ denote the eigenvalues of $-\Delta$ under Neumann boundary conditions.
Applying $\mathcal{M}^{-1}$ to a vector $v$ is implemented by reshaping $v$ to an $n\times n$ array, computing its 2D DCT, dividing entrywise by $\Lambda + k^2(1+\bar x)$, and transforming back via the inverse 2D DCT.
This results in an efficient $\mathcal{O}(n^2\log n)$ preconditioner application that substantially accelerates GMRES convergence.

\paragraph{Proposal for \ProxMH.} The objective of the proximal correction in this multiple-source setting is the following: let $\xapprox \sim \pi_a$, the proposal of \ProxMH is
\begin{align}\label{eq:prox-multi-source}
    x\gets \argmin_{x} \sum_{i\in[n_s]}\left\|\A_i(x) - \Aapprox_i (\xapprox)\right\|^2 + \beta \|x-\xapprox\|^2.
\end{align}

\paragraph{One Gauss--Newton step formulation for \cref{eq:prox-multi-source}.} A single Gauss--Newton step to optimize \cref{eq:prox-multi-source} initialized at \(\xapprox\) is given by
\begin{align}
\begin{cases}
    &x \gets \xapprox + \delta x,\\
    \left(
\sum_{i=1}^{n_s} \J_i(\xapprox)^\top \J_i(\xapprox)
+ \beta \mathbf{I}
\right)
&\delta x = -\sum_{i=1}^{n_s} \J_i(\xapprox)^\top \left(
\A_i(\xapprox) - \Aapprox_i(\xapprox) \right),
\end{cases}
\end{align}
where $\J_i(\xapprox) := \left.\dfrac{\partial \A_i(x)}{\partial x}\right|_{x=\xapprox}$ 
denotes the Jacobian of $\A_i$ evaluated at $\xapprox$.

\begin{table}[t]
\caption{Dimensions of parameters in the nonlinear Helmholtz test in \cref{sec:test-nonlinear}.}
\centering
\begin{tabular}{cccc}
\toprule
$d_x$ &$d_u$ & $\widetilde{d}_u$ &$d_y$\\
\midrule
1024 {\color{gray}($=32^2$)} & 16,384 {\color{gray}($=128^2$)} & 4,096 {\color{gray}($=64^2$)} & 512\\
\bottomrule
\end{tabular}
\label{tab:dimension-nonlinear}
\end{table}

\begin{figure}[t]
    \centering
    \includegraphics[width=0.4\linewidth]{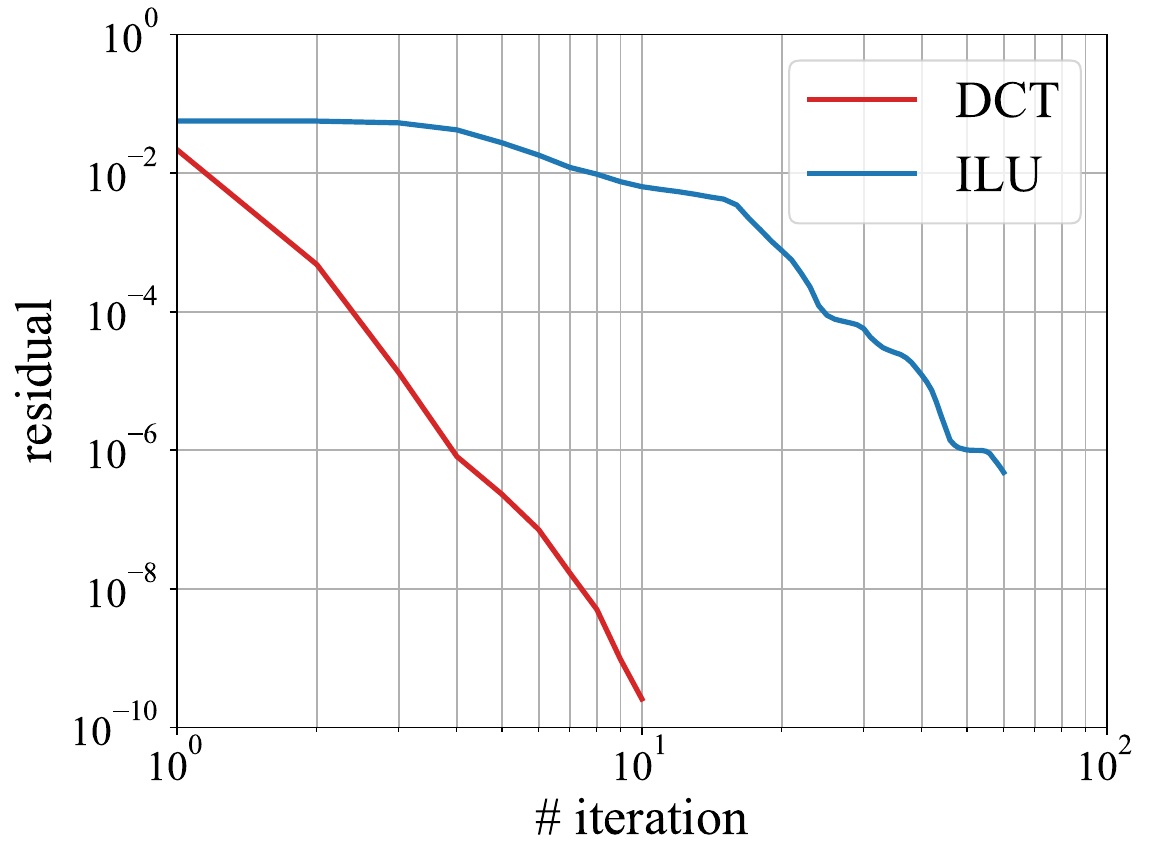}
    \caption{Effect of preconditioning on GMRES convergence for the nonlinear Helmholtz test.}
    \label{fig:gmres}
\end{figure}

\begin{figure}[t]
\centering
\begin{tikzpicture}[
  rowlabel/.style={
    rotate=90,
    anchor=center,
    text width=10mm,
    align=center,
    font=\scriptsize\bfseries,
    inner sep=1pt
  },
  toplabel/.style={
    font=\small\bfseries,
    align=center
  }
]

\def\W{0.6\columnwidth}  
\def\ygap{2mm}           
\def\xlab{-7mm}          
\def\ytitle{3mm}         

\node[inner sep=0pt] (f1)
  {\includegraphics[width=\W]{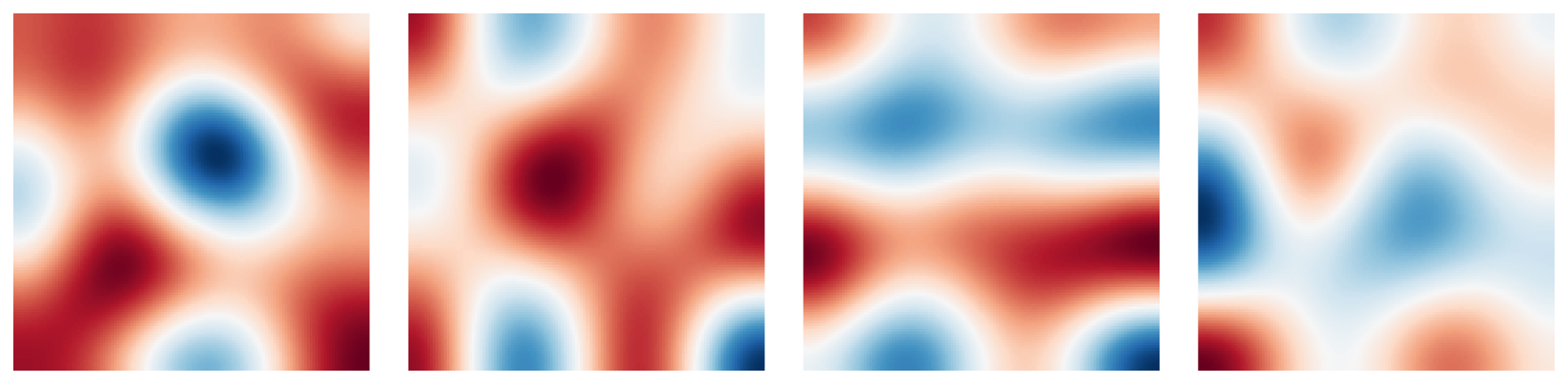}};
\node[rowlabel] at ([xshift=\xlab]f1.west) {Ground Truth};

\node[toplabel] at ([xshift=14mm,yshift=\ytitle]f1.north west) {$u_1$};
\node[toplabel] at ([xshift=8mm,yshift=\ytitle]$(f1.north west)!0.33!(f1.north east)$) {$u_2$};
\node[toplabel] at ([xshift=-2mm,yshift=\ytitle]$(f1.north west)!0.66!(f1.north east)$) {$u_3$};
\node[toplabel] at ([xshift=-15mm,yshift=\ytitle]f1.north east) {$u_4$};

\node[inner sep=0pt, below=\ygap of f1] (f2)
  {\includegraphics[width=\W]{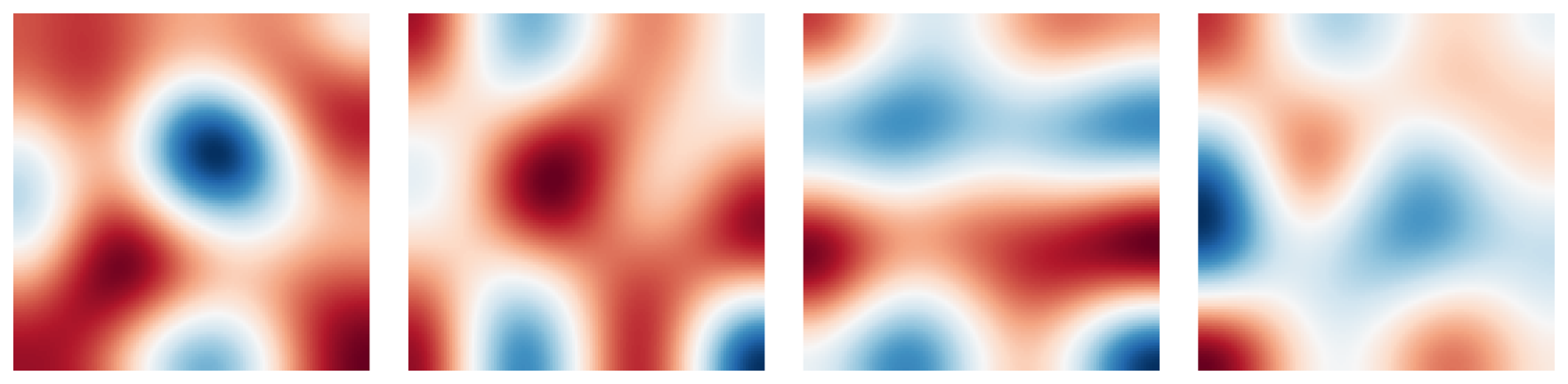}};
\node[rowlabel] at ([xshift=\xlab]f2.west) {MAP};

\node[inner sep=0pt, below=\ygap of f2] (f3)
  {\includegraphics[width=\W]{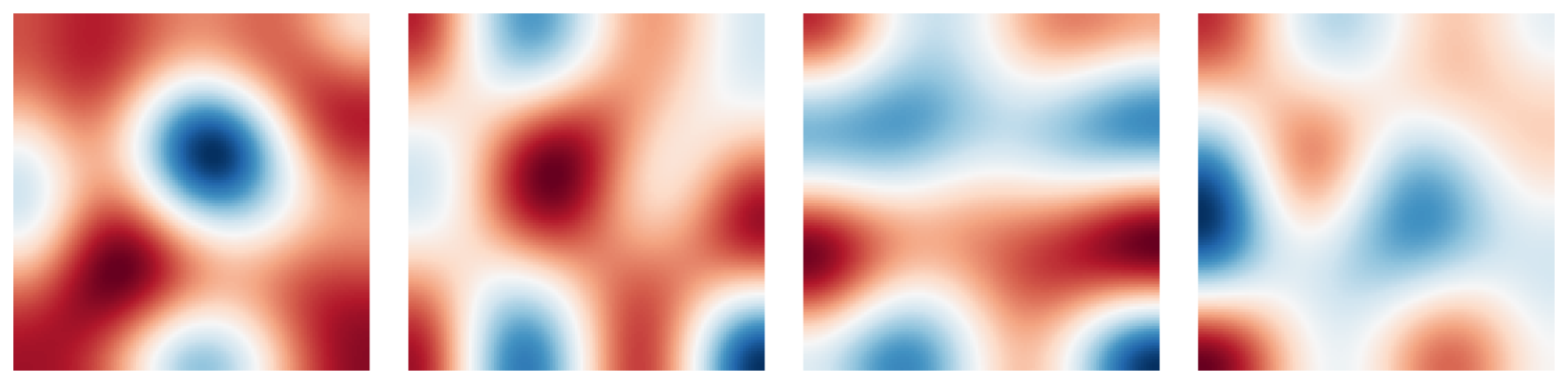}};
\node[rowlabel] at ([xshift=\xlab]f3.west) {Posterior Mean};

\end{tikzpicture}
\caption{Wavefields $\{u_i = \mathcal{L}(x)^{-1} f_i\}_{i=1}^4$ for the nonlinear Helmholtz test.
Each column corresponds to a distinct source $f_i$.
\textbf{Top:} Ground-truth wavefields.
\textbf{Middle:} Wavefields obtained from the MAP estimate.
\textbf{Bottom:} Posterior mean wavefields computed from posterior samples.}
\label{fig:u-plots}
\end{figure}
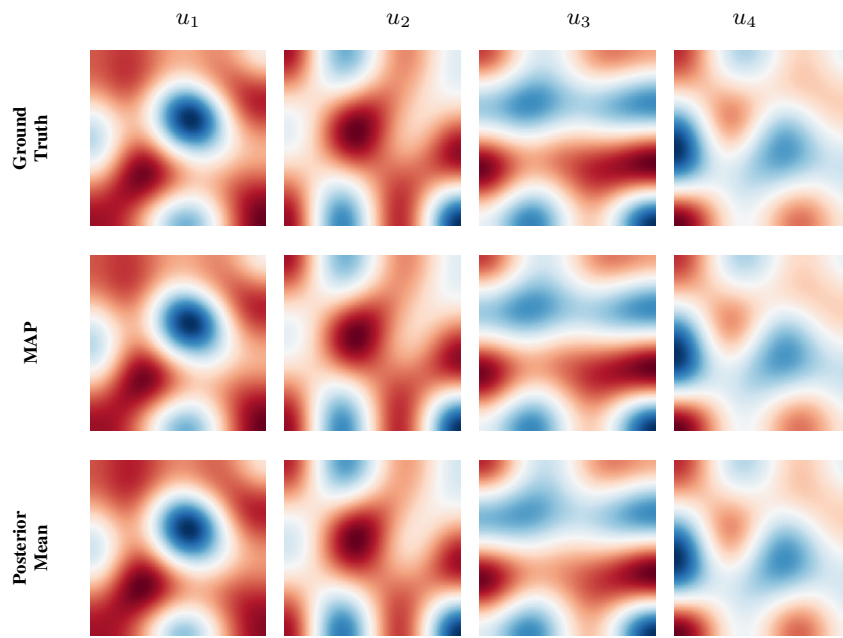

\end{document}